\documentclass[table,xcdraw]{article}
\usepackage[geometry]{ifsym}

\usepackage{pgf}
\usepackage{amsmath,amssymb}
\usepackage{mathrsfs}
\usepackage{bbm}
\usepackage{verbatim}
\usepackage{epsfig}
\usepackage{graphicx}
\usepackage{color}
\usepackage{import}
\usepackage{bussproofs}
\usepackage{qtree}
\usepackage{tree-dvips}
\usepackage{lingmacros}
\usepackage{amsfonts}
\usepackage{mathdots}
\usepackage{amssymb}
\usepackage{pifont}
\usepackage[safe]{tipa}
\usepackage{newlfont}
\usepackage{latexsym}
\usepackage{multirow}
\usepackage{array,cmll}
\usepackage{textcomp}
\usepackage{accents}
\usepackage{rotating}
\usepackage{logreq}

\usepackage{hyperref}

  \usepackage{amsmath}
  \usepackage{amsthm}
  \usepackage{latexsym}
  \usepackage{amscd}
  \usepackage{amssymb}
  \usepackage{bussproofs}
  \usepackage{txfonts}

  \usepackage{tikz}
  \usepackage[position=top]{subfig}

  \usepackage{leqno}
  \usepackage{txfonts}
  \usepackage{tikz}

\EnableBpAbbreviations
\usepackage{url}
\usepackage{authblk}
\usepackage{amsmath}
\usepackage{amsfonts}
\usepackage{longtable}
\usepackage{color}
\usepackage[T1]{fontenc}
\usepackage[all]{xy}
\usepackage{amssymb}
\usepackage{wrapfig}
\usepackage[left=3cm,top=3cm,right=3cm,bottom=2cm]{geometry}
\usepackage{graphicx}
\usepackage{cmll}
\usepackage{etex}
\usepackage{graphics}
\usepackage{color}
\usepackage{cancel}
\usepackage{stmaryrd}
\usepackage{amsmath}
\usepackage{amsthm}
\usepackage{amssymb,stmaryrd}
\usepackage{proof}
\usepackage{verbatim,cmll}
\usepackage{setspace}
\usepackage{lscape}
\usepackage{latexsym,relsize}
\usepackage{amscd}
\usepackage{multicol}
\usepackage{bussproofs}
\usepackage[position=top]{subfig}
\usepackage{leqno}
\usepackage{tikz}
\usetikzlibrary{arrows,decorations.pathmorphing,backgrounds,positioning,fit,petri}
\usepackage{multicol}
\usetikzlibrary{arrows}
\usetikzlibrary{matrix}
\usetikzlibrary{patterns}
\usetikzlibrary{shapes}

\newcommand{\bel}{\mathrm{bel}}
\newcommand{\pl}{\mathrm{pl}}
\newcommand{\mass}{{m}}

\renewcommand{\emph}{\textbf}

\newcommand{\jty}{J^{\infty}}
\newcommand{\mty}{M^{\infty}}

\newcommand{\nomi}{\mathbf{i}}

\newcommand{\marginnote}[1]{\marginpar{\raggedright\tiny{#1}}}

\newcommand{\val}[1]{[\![{#1}]\!]}
\newcommand{\descr}[1]{(\![{#1}]\!)}
\renewcommand{\phi}{\varphi}

\newcommand{\commment}[1]{}

\def\aol{\rule[0.5865ex]{1.38ex}{0.1ex}}

\def\pdra{\mbox{$\,>\mkern-8mu\raisebox{-0.065ex}{\aol}\,$}}
\def\pdla{\mbox{\rotatebox[origin=c]{180}{$\,>\mkern-8mu\raisebox{-0.065ex}{\aol}\,$}}}

\def\mANDORatom#1{\hbox{\hbox to 0pt{$#1\TriangleUp$\hss}$#1\TriangleDown$}}

\newcommand{\mcAND}{%
\mathrel{\ooalign{\raisebox{-0.39ex}{$\mbox{\TriangleUp}$}\cr\kern4.2pt{\raisebox{-0.13ex}{$\cdot$}}}}}
\newcommand{\mcand}{%
\mathrel{\ooalign{$\vartriangle$\cr\kern1.99pt{\raisebox{-0.17ex}{$\cdot$}}}}}

\newcommand{\mand}{\vartriangle}

\newcommand{\nAND}{%
\mathrel{\ooalign{$\mbox{\TriangleUp}$\cr\kern0pt$\mbox{\rotatebox[origin=c]{180}{\TriangleUp}}$}}}

\newcommand{\nand}{%
\mathrel{\ooalign{$\vartriangle$\cr\kern0pt$\triangledown$}}}

\newcommand{\mcBAND}{%
\mathrel{\ooalign{\raisebox{-0.39ex}{$\mbox{\FilledTriangleUp}$}\cr\kern4.2pt{\raisebox{-0.13ex}{${\color{white}\cdot}$}}}}}

\newcommand{\mcband}{%
\mathrel{\ooalign{$\blacktriangle$\cr\kern1.99pt{\raisebox{-0.17ex}{${\color{white}\cdot}$}}}}}

\newcommand{\mcRA}{%
\mathrel{\ooalign{
                  \raisebox{-0.3ex}{$\rotatebox[origin=c]{-90}{$\mbox{{\TriangleUp}}$}$}
                                                                            \cr\kern2.7pt{\raisebox{0.2ex}{$\cdot\mkern1.3mu$}}}}}

\newcommand{\mcra}{%
\mathrel{\ooalign{$\,{\vartriangleright\,}$\cr\kern3pt{\raisebox{0ex}{$\cdot$}}}}}

\newcommand{\mcraline}{%
-{\mkern-6mu{\mathrel{\ooalign{$\,{\vartriangleright\,}$\cr\kern3pt{\raisebox{0ex}{$\cdot$}}}}}}}

\newcommand{\mdraline}{%
{\mathrel{\ooalign{$\,{\vartriangleright\,}$\cr\kern3pt{\raisebox{0ex}{$\cdot$}}}}}{\mkern-6mu}-}

\newcommand{\cra}{%
\mathrel{\ooalign{$\,-{\mkern-3mu\vartriangleright\,}$\cr\kern8pt{\raisebox{0ex}{$\cdot$}}}}}

\newcommand{\mcBRA}{%
\mathrel{\ooalign{
                  \raisebox{-0.3ex}{$\rotatebox[origin=c]{-90}{$\mbox{\FilledTriangleUp}$}$}
                                                                            \cr\kern2.7pt{\raisebox{0.2ex}{${\color{white}\cdot}$}}}}}

\newcommand{\mcbra}{%
\mathrel{\ooalign{$\,-{\mkern-3mu\blacktriangleright\,}$\cr\kern8pt{\raisebox{0ex}{$\cdot$}}}}}

\newcommand{\mcLA}{%
\mathrel{\ooalign{
                  \raisebox{-0.3ex}{$\rotatebox[origin=c]{90}{$\mbox{\TriangleUp}$}$}
                                                                                     \cr\kern5.5pt{\raisebox{0.2ex}{$\cdot$}}
                                                                                                                              }}}

\newcommand{\mcla}{%
\mathrel{\ooalign{$\,{\vartriangleleft\,}$\cr\kern5pt{\raisebox{0ex}{$\cdot$}}}}}

\newcommand{\mclaline}{%
-{\mkern-6mu{\mathrel{\ooalign{$\,{\vartriangleleft\,}$\cr\kern5pt{\raisebox{0ex}{$\cdot$}}}}}}}

\newcommand{\mcBLA}{%
\mathrel{\ooalign{
                  \raisebox{-0.3ex}{$\rotatebox[origin=c]{90}{$\mbox{\FilledTriangleUp}$}$}
                                                                                     \cr\kern5.5pt{\raisebox{0.2ex}{${\color{white}\cdot}$}}
                                                                                                                                            }}}

  \numberwithin{equation}{section}
\newcommand{\gI}{%
\mathrel{\ooalign{$\mbox{T}$\cr\kern0pt$\mbox{\rotatebox[origin=c]{180}{T}}$}}}

\newcommand{\apdla}{\pdla'}

\def\aol{\rule[0.5865ex]{1.38ex}{0.1ex}}

\makeatletter
\newcommand{\WKnowProxy}[2]{%
  {\mathbin{\ooalign{$#1\circ#2 $\cr\hidewidth
   \raise.155ex\hbox{$#1{\scriptstyle{\ast}}#2$}\hidewidth\cr  }}}}
\makeatother

\makeatletter
\newcommand{\BKnowProxy}[2]{%
  {\mathbin{\ooalign{$#1\bullet#2 $\cr\hidewidth
   \raise.155ex\hbox{$#1{\scriptstyle{\color{white}{\ast}}}#2$}\hidewidth\cr  }}}}
\makeatother









\usepackage[utf8]{inputenc}

\newcommand{\diamdot}{\Diamond\!\!\!\!\cdot\ }
\newcommand{\diamdotb}{\Diamondblack\!\!\!\!{\color{white}{\cdot\ }}}
\newcommand{\boxdotb}{\blacksquare\!\!\!{\color{white}{\cdot\,\,}}}
\def\aol{\rule[0.5865ex]{1.38ex}{0.1ex}}
\def\pdla{\mbox{\rotatebox[origin=c]{180}{$\,>\mkern-8mu\raisebox{-0.065ex}{\aol}\,$}}}
\def\pdra{\mbox{$\,>\mkern-8mu\raisebox{-0.065ex}{\aol}\,$}}
\newcommand{\br}{\mbox{$\,\vartriangleright\mkern-15mu\rule[0.51ex]{1ex}{0.12ex}\,$}}
\newcommand{\abr}{\mbox{$\blacktriangleright\mkern-15.5mu\textcolor{white}{\rule[0.51ex]{1.2ex}{0.12ex}}$}}
\theoremstyle{plain}
\newtheorem{thm}{Theorem}

\newtheorem{cor}[thm]{Corollary}
\newtheorem{prop}[thm]{Proposition}

\newtheorem{lemma}[thm]{Lemma}

\theoremstyle{definition}

\newtheorem{definition}[thm]{Definition}

\newtheorem{example}[thm]{Example}

\newtheorem{remark}[thm]{Remark}

\title{Flexible categorization for auditing using formal concept analysis and Dempster-Shafer theory}
\author[5,6]{Marcel Boersma}
\author[1]{Krishna Manoorkar\footnote{Krishna Manoorkar is supported by the NWO grant KIVI.2019.001 awarded to Alessandra Palmigiano.}}
\author[1,2]{Alessandra Palmigiano}
\author[1]{Mattia Panettiere}
\author[1]{Apostolos Tzimoulis}
\author[3,4]{Nachoem Wijnberg} 
\affil[1]{School of Business and Economics, Vrije Universiteit, Amsterdam, The Netherlands}
\affil[2]{Department of Mathematics and Applied Mathematics, University of Johannesburg, South Africa}
\affil[3]{College of Business and Economics, University of Johannesburg, South Africa}
\affil[4]{Faculty of Economics and Business, University of Amsterdam, The Netherlands}
\affil[5]{Computational Science Lab, University of Amsterdam, Amsterdam, The Netherlands}
\affil[6]{KPMG, Amstelveen, The Netherlands}

\date{}

\begin{document}

\maketitle
\begin{abstract}
   Categorization of business processes is an important part of auditing. Large amounts
of transnational data in auditing can be represented as transactions between financial accounts using
weighted bipartite graphs. We view such bipartite graphs as many-valued formal contexts, which we use to obtain explainable categorization of these business processes in terms of financial accounts involved in a business process by using methods in formal concept analysis. The specific  explainability feature of the methodology introduced in the present paper
provides several advantages over e.g.~non-explainable machine learning techniques, and in fact, it can be taken as a basis for the development of  algorithms which perform the task of clustering on transparent and accountable principles. Here, we focus on obtaining and studying different ways to categorize  according to different extents of interest in different financial accounts, or {\em interrogative agendas},  of various agents or sub-tasks in audit. We use Dempster-Shafer mass functions to represent agendas showing different interest in different set of financial accounts. We propose two new methods
to obtain categorizations from these agendas. We also model some possible deliberation scenarios between agents with different interrogative agendas to reach an aggregated agenda and categorization. The framework
developed in this paper provides a formal ground to obtain and study explainable categorizations from the data represented as bipartite graphs according to the agendas of different agents in an organization (e.g.~an audit firm), and interaction between these through deliberation.\\
    {\em Keywords: Auditing, Categorization, Formal Concept Analysis, Demspter-Shafer theory, Interrogative Agendas} 
\end{abstract}
\section{Introduction}

{\em Financial auditing} is the process of examining and providing an independent, third-party, expert opinion on the truth and fairness of financial information being presented by a company, and the compliance of this information with applicable accounting standards and relevant legislation.
Auditors collect, weight, and combine information in formulating their judgment about the truth/fairness of their clients’ financial statements. At the core of auditing is the {\em qualitative} evaluation of the aggregate of all the relevant data of the financial report of the given firm. Auditors exercise professional judgment in determining the type and extent of information to collect, and in assessing the implications of this information. To form their judgement, auditors employ their training, experience, industry knowledge, and other forms of information-gathering that may not always be uniformly recorded.
A financial audit is broken up into components, corresponding to the individual entities belonging to the consolidated company (e.g.~legal entities, branches, departments), as well as separate business cycles (e.g.~order-to-cash, procure-to-pay, financing, payroll), their related balance sheet and income statement accounts, and the company's assertions about these accounts (e.g.~their completeness, existence, accuracy, valuation, ownership, and their correct presentation in the financial statements).
To collect evidence pertaining to each assertion, auditors choose a mixture of procedures on each component. These procedures range from risk assessment, the testing of design and operating effectiveness of internal controls, analytical procedures, statistical and non-statistical sampling and other detail testing procedures. Procedures may be geared to all aggregation levels, ranging from the entity as a whole to a specific assertion of a specific part of a financial statement account. The objective of the auditor then is to {\em aggregate} all evidence obtained, both confirming and disconfirming, to issue an opinion as to whether the financial statements present fairly, in all material respects, the financial position of the company as of the balance sheet date and the results of its operations and its cash flows for the year then ended in accordance with generally accepted accounting principles in the country where the report is issued. Because precise guidelines for information, collection, and evaluation in auditing do not exist, individual and collective professional judgment plays a key and pervasive role in auditing \cite{joyce1976expert}.

\paragraph{State-of-the-art.} Text-analysis techniques, AI and data-analytic methods have already been very useful to speed up the controlling function of auditors by flagging low-level, individual anomalies (e.g.~entries with keywords of a questionable nature, entries from unauthorized sources, or an unusually high number of journal entry postings just under authorized limit) \cite{kokina2017emergence, Raphael2017RethinkingTA, dickey2019machine, wang2019reflections}. 
However, even in the absence of {\em low-level} anomalies, it is an auditor’s task to detect {\em higher-level} anomalies concerning the {\em coherence} of all information at different aggregation levels of the financial information, when combined together: indeed, one and the same event or action in the life of a company is witnessed by evidence at multiple levels (e.g.~at the overall financial statement level, at the account level, at the transactions stream level, and at the individual assertion level of accounts). Also, certain items of evidence pertain to multiple assertions of an account. For example, confirmation of accounts receivables pertains to the “Existence” and “Valuation” assertions of the accounts receivable balance. The core of an auditor’s work is to provide a higher-level, qualitative judgment on whether all these different pieces of evidence lead to a coherent picture of any given event, and hence to a fair presentation of the company’s financial information. 

\paragraph{Challenges.} The extant research on the formal foundations of the reasoning at the core of auditing \cite{shafer1990bayesian, srivastava1992belief, srivastava1993belief, desai2006analytical, sun2006information, gao2011evidential, srivastava2011introduction, desai2017external, mock2018using} has highlighted several inadequacies of probability theory for representing uncertainty in audit judgments, the most prominent of which are the logical implications stemming from the complementarity of probabilities, and the ensuing difficulty in drawing the distinction—critical to the practice of auditing—between the absence of evidence in support of a statement and the presence of evidence to its contrary. This literature has advocated the use of {\em Dempster-Shafer’s theory} of belief functions \cite{dempster2008upper, shafer1976mathematical} to overcome this problem. However, the extant literature has not yet developed specific formal models for audit on which machine learning algorithms can be designed and tested. This is the first issue which this paper starts  to address. 
Moreover, as pointed out earlier, most current AI techniques focus on detecting low-level anomalies taken in isolation, while there is not yet much work in AI specifically tailored to assist auditors on their higher-level tasks. As we will argue, the contributions of the present paper set the stage for addressing also this second aspect.

\paragraph{Aims and contributions.}
The present paper starts a line of research aimed at developing formal models specifically designed to  analyse and represent the higher-level processing of information of  experienced auditors, and at using this formal understanding as a base to develop data-analytic tools specifically designed to flag higher-level anomalies.

Specifically, we posit that the experienced auditors’ evaluation of evidence is rooted in a process of {\em category-formation}, by which pieces of evidence are clustered together in categories which are possibly very different from the ``natural'' or ``official'' categories with which the evidence is presented in the self-reported financial statement of the given company under examination. These categories provide the context of evaluation in which different and possibly very heterogeneous pieces of evidence are compared with/against each other, and their overall coherence is evaluated via this comparison.\\
{\em Background theory: Formal Concept Analysis. } Accordingly, the formal framework presented in this paper is set within the many-valued counterpart of Formal Concept Analysis \cite{wille1996formal, belohlavek1999fuzzy}, in which (vague) categories are represented as the formal concepts associated with bipartite weighted graphs seen as {\em fuzzy formal contexts} \cite{wille1996formal, belohlavek1999fuzzy}, each consisting of two domains $A$ (of objects) and $X$ (of features), and a weighted relation $I$ between them.

Besides the fact that, mathematically, both bipartite graphs  (whether weighted or not, directed or not) and formal contexts are 
isomorphic structures\footnote{Since the objects and features of a formal context can be seen as disjoint vertex sets of a bipartite graph respectively, and the incidence relation between the two sets can be seen as the set of edges.}, they are both used to represent data \cite{wille1996formal, vskopljanac2014formal, kuznetsov2004machine, pavlopoulos2018bipartite, hayes2004bipartite, ravasz2002hierarchical, newman2001scientific, cobb2003application}, and in particular bipartite graphs have been used in  \cite{pavlopoulos2018bipartite, hayes2004bipartite, ravasz2002hierarchical, newman2001scientific, cobb2003application}. Regarding bipartite graphs as formal contexts allows for access to the well-known representation of the  hierarchy of formal concepts associated with each formal context \cite{birkhoff1940lattice}, and hence  provides a formally explicit   categorization process associated with modelling data as bipartite graphs.   
 Formal concept analysis  provides:
\begin{enumerate}
    \item  a hierarchical, rather than flat, categorization of objects.
    \item categories that have a double (i.e.~both extensional and intensional) representation, which allows to make them  explainable, in the sense that each category can be effectively reconstructed in terms of its   features. 
    \item  a more structured control of categorization based on the generation of categories from arbitrary subsets of objects or of features.  That is, the we can add or remove objects or features from the categorization in straight forward manner and study or explain resulting categorization in a formal way. 
    \item the basis for a formal framework for addressing and supporting vagueness, epistemic uncertainty, evidential reasoning, and incomplete information \cite{conradie2017toward, conradie2021rough, frittella2020toward}.
\end{enumerate}
Besides these advantages in categorization, this conversion would allow us to access tools in formal concept analysis developed for several applications like knowledge discovery and management, information retrieval, attribute exploration \cite{priss2006formal, qadi2010formal, poelmans2010formal, valtchev2004formal, poelmans2013formal, ganter2012formal, wille1996formal}. 

 Clustering nodes of one or both types in a bipartite graph is an important problem in several fields \cite{zha2001bipartite, xu2013behavior,boersma2018financial,boersma2020reducing, xu2016interactive, gaume2013clustering}, and there have been several approaches  to address this task  \cite{schaeffer2007graph, zha2001bipartite,gaume2013clustering}.  
One particularly relevant example for the present paper concerns clustering of nodes of a bipartite graph which represents the network of financial transactions (i.e.~a {\em financial statements network}). Financial statements networks \cite{boersma2018financial} are bipartite graphs  $G=(A, X, I)$,  where $A$ is the set of business processes, $X$ is the set of financial accounts (e.g.~tax, revenue, trade receivables),  and $I : A \times X\to [-1, 1]$ is  such that, for any business process $a$ and any financial account $x$, the value of $I(a, x)$ is the share of $a$ into $x$, where the (positive or negative) sign of $I(a, x)$ represents whether money is credited into or debited from $x$ in $a$, respectively. By a {\em business process}, here we mean a set of credit and debit activities meant to produce a specific output \cite{boersma2018financial}. For example, all transaction relating to a sale of an object constitute a business process. Given the journal entry data we can map a collection of records to a particular business process. 
Any business process $a$  can then be described in terms of  the values of $I(a, x)$ for any $x\in X$. This description yields a  categorization of business processes which can be of interest in several applications \cite{boersma2018financial, boersma2020reducing}. To make processing and interpreting contexts easier,  many-valued formal contexts $G=(A, X, I)$ as above may be converted into two-valued formal  contexts using conceptual scaling \cite{ganter1989conceptual}.

{\em Interrogative agendas. } As mentioned above, the final outcome of the auditing process is the formation of a qualitative opinion, by expert auditors, on the fairness and completeness of a given firm's financial accounts. Towards the formation of their opinion, not all features have the same weight in the eyes of the auditors: indeed,  different auditors  may have different views about which parts of the financial accounts they consider more relevant to the formation of their opinion. For example, an auditing task may be subdivided among different auditors focusing on different tasks. Auditors doing a specific task may be much more interested in some specific financial accounts than others. For example an auditor may be focusing on the specific task relating more to the 
books of  the tax department, and have more interest in the features like   revenue, tax, and fixed assets while for another auditor focusing on a task for which business processes relating to the  human resources department are of more importance, the features of interest may be other expenses and personnel expenses. In the case, a certain set of features has significantly more relevance to an agent (auditor) we can model this situation by setting this set of features to be the agenda or features of interest of that agent. In many cases, the agenda of an agent  may not be  realistically approximated in such a simple way but might consist of different relevance or importance values assigned to the different set of features. In such cases, we would use Dempster-Shafer mass functions to represent such agendas.

Another reason why features do not have the same weight or importance to an auditor 
is connected with the different degrees of risk in the functioning of certain departments of a firm. Different departments or groups of transactions may have different perceived degree of risk leading an auditor to assign different importance to different financial accounts involved. 

Similarly to the different epistemic attitudes entertained by experts auditors, also different procedures used in auditing may focus unequally on different features  (financial accounts). For example, some methods may be more accurate in detecting inconsistencies in revenue data compared to detecting inconsistencies in personal expenses data, or vice-versa. In such situations, categorizations of business processes based on different features may be useful in understanding which methods are more effective for which business processes, choosing samples to be used in different processes and analyzing results in decision-making. This may also be helpful in categorization of errors obtained from such methods, a topic of significant interest in auditing \cite{giriunas2014evaluation, singh2019data}.
In the proposed  framework, the different epistemic attitudes of expert auditors are formalized by the notion of {\em interrogative agenda} (or research agenda\cite{enqvist2012modelling}).
The framework introduced in the present paper may be also used to provide a better identification of the anomalies causing given business processes to be classified as   fraudulent or inconsistent.  Such categorizations may be very useful in understanding source and extent of concerns and in making decision about which additional data may be needed in decision-making. Besides introducing a framework for such categorizations, the present work also discusses methods for formalizing how all these different classification methods interact, which may be useful not only in understanding the interactions between these processes, but also  for pooling their outcomes together in decision-making.

{\em A logical framework. } So far, we have described our first contributions, starting from considering financial statements networks as formal contexts; we have also discussed that this approach provides a natural and structured way to categorize business processes based on different agendas. This approach also allows us to consider other features (e.g.~time of transactions, value, location) which might not be present in the network itself. There have been some attempts in the past to study the formal contexts obtained by restricting to  subsets of features to obtain formal contexts of interest \cite{cole1999scalability}. In this work, we introduce a  logical framework specifically designed to systematically represent and reason about the different ways business processes  can be categorized on the basis of different subsets of features selected on the basis of the different epistemic attitudes of agents (expert auditors), as well as the interaction between and possible aggregation of these various categorizations, by means of a deliberation process. 
 The interaction between agents (auditors), their agendas, objects, and categories can be further augmented with several useful notions, including preference orders, similarities, influences, and dependencies among features.

{\em Preference-aggregation via generalized  Dempster-Shafer theory.  } We extend the logical framework described above to the cases where the different interrogative agendas of the various agents  might induce  different priorities over the set of features  used for categorization. Such priorities can be represented as Dempster-Shafer mass functions on the set of agendas. These mass functions on agendas also induce  mass functions on different categorical hierarchies (concept lattices), and represent  priorities for each categorization based on these agendas. Dempster-Shafer theory has been applied to model reasoning under uncertainty about categorization of objects and features, or for describing preferences for certain categories in a given categorization  \cite{frittella2020toward}. In the same works, a Dempster-Shafer mass function is defined over a given concept lattice  representing evidence for a set of objects or features belonging to a category or preference in that category.
Here, we take a different approach, and use Dempster-Shafer theory to describe the priority or importance of different set of features in  categorization using Dempser-Shafer mass functions. That is, mass functions are used to choose which categorization is relevant to a given task. We then try to formalize deliberation between the agents with different agendas represented by mass functions using aggregation rules in Dempster-Shafer theory.

We use the terms crisp and non-crisp in this paper to talk about agendas given by a set of features, and a Dempster-Shafer mass function on the power-set of features respectively. The  term crisp signifies the fact that model has a fixed set of agendas and thus leads to a single categorization. While, the term non-crisp highlights  that the Dempster-Shafer mass functions give different priority or preference values to different sets of features (agendas) and leads us to a mass function on different possible categorizations.

\paragraph{ Structure of the paper.} In Section \ref{sec:Example}, we describe a small financial statements network and show how different set of features give different categorizations of business process in this small network. In Section \ref{sec:Prelim}, we give preliminaries on formal concept analysis, Dempster-Shafer theory, interrogative agendas, and financial statements networks. In Section \ref{sec:Interrogative agendas, coalitions, and categorization}, we describe our logical framework for reasoning with different interrogative agendas and categorizations obtained from them. In Section \ref{sec:Deliberation and categorization- crisp case}, we use  logical  framework developed in Section \ref{sec:Interrogative agendas, coalitions, and categorization} to model  possible deliberation scenarios among different agents. In Section \ref{sec:Non-crisp interrogative agendas}, we describe non-interrogative agendas given by Dempster-Shafer mass functions and categorizations obtained from these. We also propose  methods to obtain a single crisp categorization approximating this non-crisp categorization. In Section \ref{sec:Deliberation and categorization-non crisp case}, we try to formalize possible deliberation scenarios among different agents having agendas represented by Dempster-Shafer mass functions. Finally, in Section \ref{sec:Deliberation and categorization-non crisp case}, we revisit the  financial statements network described in Section \ref{sec:Examples end} and describe categorizations obtained from it under different crisp and non-crisp agendas. Finally in Section \ref{sec:Conclusion and further directions}, we give our conclusions and mention several directions for future research. 

\section{Example}\label{sec:Example}
In this section, we informally illustrate the ideas  discussed in the introduction by way of a toy example. Consider the financial statements network 
represented in  Table \ref{database table}, with  business processes $\{ a_1, a_2, \ldots, a_{12}\}$ and   financial accounts $\{x_1, x_2, \ldots, x_6\}$ specified as follows:

\begin{center}
    \begin{tabular}{|c|c|c|c|}
    \hline
      $x_1$ & tax & $x_2$ & revenue \\
      \hline
       $x_3$&cost of sales & $x_4$& personnel expenses \\
       \hline
      $x_5$& inventory & $x_6$ &other expenses\\
      \hline
    \end{tabular}
\end{center}
As discussed above,  each cell of the Table \ref{database table} reports the value of the (many-valued) relation $I: \{a_1,\ldots, a_{12}\}\times \{x_1,\ldots, x_6\}\to [-1, 1]$,  which, for any business process $a$ and account $x$, represents the share of $a$ in $x$. 

Let $j_1$, $j_2$, and $j_3$ be agents with 
different agendas. Specifically, agent $j_1$ is interested in the financial accounts $x_1$  (tax), $x_2$ (revenue), and $x_5$ (inventory), agent $j_2$ in $x_1$  (tax), $x_2$ (revenue), and $x_3$ (cost of sales), while agent $j_3$ is interested in $x_1$  (tax),  and $x_3$ (cost of sales). The various ways of categorizing $\{a_1, a_2, \ldots, a_6\}$ (i.e.~forming concept lattices)  under these different agendas by using interval scaling \footnote{{Interval scaling is one of the methods used commonly for conceptual scaling. For more, see \cite{ganter1989conceptual}.}} with 5 intervals of equal length between $[-1,1]$ are shown in the following diagrams. It is clear that the categorizations obtained differ from each other depending on criterion used. For example, business processes $a_1,  a_3, a_5$ are indistinguishable under the agenda of $j_2$, while $a_2$ forms a singleton category under the agenda of $j_3$, and hence is distinguishable from all other business processes. The business process forming smaller categories may be considered uncommon and may be of further interest in auditing tasks. However, as shown by the above example, this is influenced by the set of features or agenda used for categorization. In Section \ref{sec:Examples end}, we will consider more examples of categorizations obtained from different agendas associated with individual agents, and the agendas obtained as outcomes of processes of deliberation. 
\begin{figure}[h]
\begin{center}
\begin{tikzpicture}
\draw[very thick] (0, 0) -- (-2,1) --(-1,2)--(0,3)--(1,2)--(2,1)--(0,0);
\draw[very thick] (0,0)--(0,1);
\draw[very thick] (-1,2)--(0,1)--(1,2);
\draw[very thick] (0,0)--(2,1);

	\draw (0,1) node[right] {1,2,3,5};
	\draw (-2,1) node [left] {4};
	\draw (2, 1) node [right] {6};
    \draw (-1,2) node [left] {1,2,3,4,5};
    \draw (1,2) node [right] {1,2,3,5,6};
   \draw (0,3) node [above] {1,2,3,4,5,6};
   	\draw (0,0) node[below] {};
  	\filldraw[black] (0,0) circle (2.5 pt);
  	 \filldraw[black] (-1,2) circle (2.5 pt);
  	\filldraw[black] (-2,1) circle (2.5 pt);
  	\filldraw[black] (0,3) circle (2.5 pt);
  	\filldraw[black] (1,2) circle (2.5 pt);
  	\filldraw[black] (2,1) circle (2.5 pt);
  	\filldraw[black] (0,1) circle (2.5 pt);

\draw[very thick] (6, 0) -- (4,1) --(5,2)--(6,3)--(7,2)--(8,1)--(6,0);
\draw[very thick] (6,0)--(6,1);
\draw[very thick] (5,2)--(6,1)--(7,2);
\draw[very thick] (6,0)--(4,1);
\filldraw[black] (6,0) circle (2.5 pt);
  	 \filldraw[black] (4,1) circle (2.5 pt);
  	\filldraw[black] (5,2) circle (2.5 pt);
  	\filldraw[black] (6,3) circle (2.5 pt);
  	\filldraw[black] (7,2) circle (2.5 pt);
  	\filldraw[black] (8,1) circle (2.5 pt);
  	\filldraw[black] (6,1) circle (2.5 pt);

	\draw (6,1) node[right] {1,3,5,6};
	\draw (4,1) node [left] {4};
	\draw (8, 1) node [right] {2};
    \draw (5,2) node [left] {1,3,4,5,6};
    \draw (7,2) node [right] {1,2,3,5,6};
   \draw (6,3) node [above] {1,2,3,4,5,6};
   	\draw (6,0) node[below] {};	
\end{tikzpicture}
\end{center}
\caption{Left: categorization obtained from the agenda of $j_1$,  Right: categorization obtained from the agenda of $j_3$.} \label{img:$j_1$}
\end{figure}
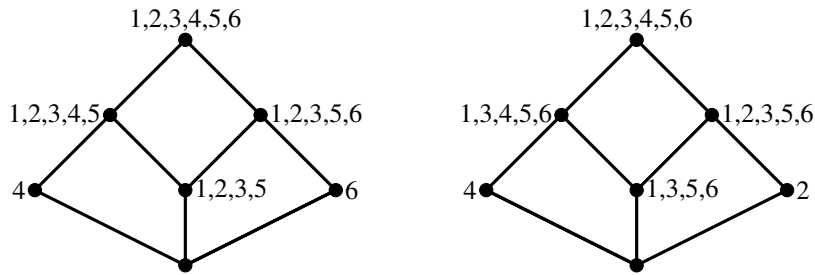

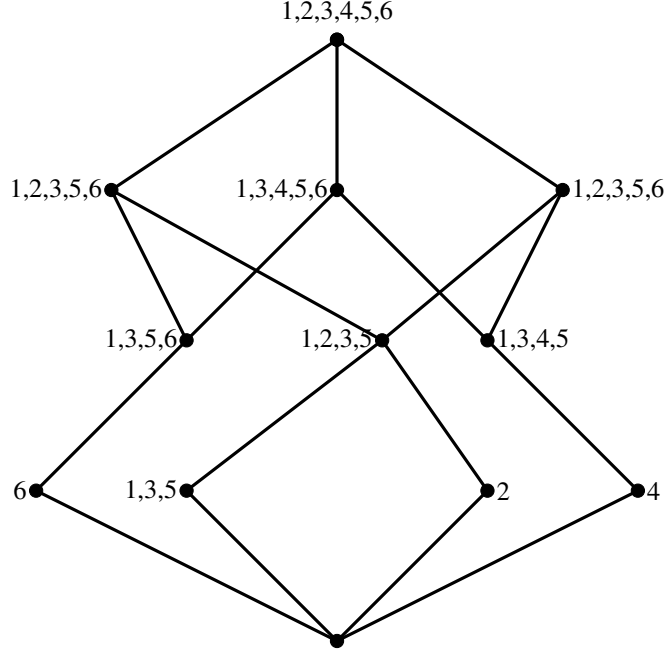
\begin{figure}[!h]

\begin{center}
\begin{tikzpicture}

\draw[very thick] (3,-1)--(0,-3)--(1,-5)--(-1,-7)--(3,-9)--(7,-7)--(5,-5)--(6,-3)--(3,-1);
\draw[very thick] (3,-1)--(3,-3)--(1,-5);
\draw[very thick] (3,-3)--(5,-5);
\draw[very thick] (3,-9)--(1,-7)--(3.6,-5)--(6,-3);
\draw[very thick] (3.6,-5)--(0,-3);
\draw[very thick] (3,-9)--(5,-7)--(3.6,-5);
	 \filldraw[black] (3,-1) circle (2.5 pt);
  	\filldraw[black] (0,-3) circle (2.5 pt);
  	\filldraw[black] (1,-5) circle (2.5 pt);
  	\filldraw[black] (-1,-7) circle (2.5 pt);
  	\filldraw[black] (3,-9) circle (2.5 pt);
  	\filldraw[black] (7,-7) circle (2.5 pt);
  	\filldraw[black] (5,-5) circle (2.5 pt);
  	\filldraw[black] (6,-3) circle (2.5 pt);
  	\filldraw[black] (3,-1) circle (2.5 pt);
  	\filldraw[black] (3,-3) circle (2.5 pt);
  	\filldraw[black] (1,-7) circle (2.5 pt);
  	  	\filldraw[black] (5,-7) circle (2.5 pt);
  		\filldraw[black] (3.6,-5) circle (2.5 pt);

\draw (3,-9) node[below] {};	
\draw(-1,-7) node [left] {6};
\draw(1,-7) node [left] {1,3,5};
\draw(5,-7) node [right] {2};
\draw(7,-7) node [right] {4};
\draw(3.6,-5) node [left] {1,2,3,5};
\draw(5,-5) node [right] {1,3,4,5};
\draw(1,-5) node [left] {1,3,5,6};
\draw(0,-3) node [left] {1,2,3,5,6};
\draw(3,-3) node [left] {1,3,4,5,6};
\draw(6,-3) node [right] {1,2,3,5,6};
\draw(3,-0.9) node[above] {1,2,3,4,5,6};
\end{tikzpicture}
\end{center}
\caption{Categorization obtained from the agenda of $j_2$.}\label{img:$j_2$}

\end{figure}
\newpage
\section{Preliminaries}\label{sec:Prelim}
\paragraph{Formal contexts and their concept lattices.}
A {\em formal context} \cite{ganter2012formal}  is a structure $\mathbb{P} = (A, X, I)$ such that $A$ and $X$ are sets, and $I\subseteq A\times X$ is a binary relation. 
Formal contexts can be thought of as abstract representations of databases, where elements of $A$ and $X$ represent objects and features, respectively, and the relation $I$ records whether a given object has a given feature. 
Every formal context as above induces maps $I^{(1)}: \mathcal{P}(A)\to \mathcal{P}(X)$ and $I^{(0)}: \mathcal{P}(X)\to \mathcal{P}(A)$, respectively defined by the assignments 
\begin{equation}
  I^{(1)}[B]: = 
\{x\in X\mid \forall a(a\in B\Rightarrow aIx)\}\quad\text{ and }\quad 
 I^{(0)}[Y] = 
\{a\in A\mid \forall x(x\in Y\Rightarrow aIx)\}.
\end{equation}
A {\em formal concept} of $\mathbb{P}$ is a pair 
$c = (\val{c}, \descr{c})$ such that $\val{c}\subseteq A$, $\descr{c}\subseteq X$, and $I^{(1)}[\val{c}] = \descr{c}$ and $I^{(0)}\descr{c} = \val{c}$. 
A subset $B \subseteq A$ (resp.\ $Y\subseteq X$) is said to be {\em closed}, or {\em Galois-stable}, if $\mathsf{Cl}_1(B)=I^{(0)}[I^{(1)}[B]]=B$ (resp.\ $\mathsf{Cl}_2(Y)=I^{(1)}[I^{(0)}[Y]]=Y$)\footnote{We will often use $\mathsf{Cl}$ instead of $\mathsf{Cl}_1$ and $\mathsf{Cl}_2$ when it is clear from the context to which closures we are referring to, i.e., when the type of the input is clear from the context.}.
The set of objects $\val{c}$ is  the {\em extension} of the concept $c$, while  the set of features $ \descr{c}$ is  its {\em intension}\footnote{The symbols $\val{c}$ and $\descr{c}$, respectively denoting the extension and the intension of a concept $c$, have been introduced and used in the context of a research line aimed at developing the logical foundations of categorization theory, by regarding formulas as names of categories (formal concepts), and interpreting them as formal concepts arising from  given formal contexts  \cite{conradie2017toward,conradie2021rough,frittella2020toward,conradie2016categories,conradie2019logic}.}. 
The set ${\mathrm{L}}(\mathbb{P})$  of the formal concepts of $\mathbb{P}$ can be partially ordered as follows: for any $c, d\in {\mathrm{L}}(\mathbb{P})$, 
\begin{equation}
c\leq d\quad \mbox{ iff }\quad \val{c}\subseteq \val{d} \quad \mbox{ iff }\quad \descr{d}\subseteq \descr{c}.
\end{equation}
With this order, ${\mathrm{L}}(\mathbb{P})$ is a complete lattice, the {\em concept lattice} $\mathbb{P}^+$ of $\mathbb{P}$. As is well known, any complete lattice $\mathbb{L}$ is isomorphic to the concept lattice $\mathbb{P}^+$ of some formal context $\mathbb{P}$ \cite{birkhoff1940lattice}. A  formal context  $\mathbb{P}$ is {\em finite} if its associated concept lattice $\mathbb{P}^+$ is a finite lattice.\footnote{Notice that if $\mathbb{P} = (A, X, I)$ is such that $A$ and $X$ are finite sets, then $\mathbb{P}^+$ is a finite lattice, but the converse is not true in general. For instance, if $\mathbb{P} = (A, X, I)$ gives rise to 
a finite lattice, then so does
$\mathbb{P}' := (A', X, I')$ where $A': = A\cup \mathbb{N}$ and $a'I' x$ iff $a'\in A$ and $aIx$.} 
 \paragraph{Dempster-Shafer theory.} Belief and plausibility functions are one proposal among others to generalise probabilities to situations in which some predicates cannot be assigned subjective probabilities.
In this section, we collect preliminaries on belief and plausibility functions on sets (for more details on imprecise probabilities and Dempster-Shafer theory see \cite{walley1991statistical,yager2008classic}).
\paragraph{Belief, plausibility and mass functions.}
\label{def:bel-func}
A {\em belief function} (cf.~\cite[Chapter 1, page 5]{shafer1976mathematical}) on a set $S$ is a map $\bel: \mathcal{P}(S)\to [0,1]$ such that 
$\bel(S)=1$,  and for every $n\in \mathbb{N}$,
\begin{equation} 
\bel (A_1 \cup . . . \cup A_n)  \geq  
\sum_{\varnothing \neq I \subseteq \{1, . . . , n\}}
(-1)^{|I|+1} \bel \left (\bigcap_{i \in I} A_i \right).
\end{equation}
A {\em plausibility function on} $S$ is a map $\pl: \mathcal{P}(S)\to [0,1]$ such that 
$\pl(S)=1$,  and for every $n\in \mathbb{N}$,
\begin{equation} 
\pl (A_1 \cup A_2 \cup ... \cup A_n)  \leq 
\sum_{\varnothing \neq I \subseteq \{1,2,...,n\} }
(-1)^{|I| +1}\pl 
\left( \bigcap_{i \in I} A_i 
\right).
\end{equation}
Belief and plausibility functions on sets are interchangeable notions: for every belief function $\bel$ as above, the assignment  $X\mapsto 1- \bel(\overline{X})$\footnote{Here $\overline X$ denotes the complement of $X$ with respect to $S$.} defines a plausibility function on $S$, and for every plausibility function $\pl$ as above, the assignment  $X\mapsto 1- \pl(\overline{X})$ defines a belief function on $S$. Let $S$ be some set.
A {\em Dempster-Shafer mass function} is a map $\mass: \mathcal{P}(S)\to [0,1]$ such that 
\begin{equation} 
\sum_{X \subseteq S} \mass (X) = 1.
\end{equation}
A {\em probability mass function}is a map $\mass: \mathcal{P}(S)\to [0,1]$ such that 
\begin{equation} 
\sum_{x \in S} \mass (x) = 1.
\end{equation}

We use mass function to mean a Dempster-Shafer mass function unless stated otherwise. 

On finite sets, belief (resp.~plausibility) functions and mass functions are interchangeable notions:  any mass function $\mass$ as above induces the belief function   $\bel_{\mass}: \mathcal{P}(S)\to [0,1]$ defined as 
\begin{equation} 
\bel_\mass(X) := \sum_{Y \subseteq X} \mass(Y) \qquad \text{ for every } X \subseteq S,
\end{equation}
and conversely, any belief function $\bel$ as above induces the mass function   $\mass_{\bel}: \mathcal{P}(S)\to [0,1]$ defined as We use standard notation and terms from Dempster-Shafer theory which can be found in any common reference for it, for example, see \cite{sentz2002combination,yager2008classic}. 
\begin{equation} 
\mass_\bel (X) := \bel(X) - \sum_{Y \subseteq X} (-1)^{|X \smallsetminus Y|} \bel(Y) \qquad \text{ for every } X \subseteq S.
\end{equation}
\begin{definition}
For any mass function $m:\mathcal{P}(X) \to [0,1]$, its associated {\em quality function} $\mathrm{q}_m$ is given as follows. For any $Y \subseteq X$, 
\[
\mathrm{q}_m(Y) =\sum_{Y \subseteq Z}m(Z).
\]
\end{definition}
\paragraph{Interrogative agendas
and their formal epistemic theory.} \label{prelim: imterrogative agenda} In epistemology and formal philosophy, an epistemic agent’s (or a group of epistemic agents’, e.g.~users’) interrogative agenda (or research agenda \cite{enqvist2012modelling}) indicates the set of questions they are interested in, or what they want to know relative to a certain circumstance (independently of whether they utter the questions explicitly). Interrogative agendas might differ for the same agent in different moments or in different contexts; for instance, my interrogative agenda when I have to decide which car to buy will be different from my interrogative agenda when I listen to a politician’s speech. In each context, interrogative agendas act as cognitive filters that block content which is considered irrelevant by the agent and let through (possibly partial) answers to the agent's interrogative agenda. Only the information the agent considers relevant is actually absorbed (or acted upon) by the agent and used e.g.~in their decision-making, in the formation of their beliefs, etc. 
Interrogative agendas can be organized in hierarchies, and this hierarchical structure serves to establish whether a given interrogative agenda subsumes another, and define different notions of “common ground” among agendas. Deliberation  and negotiation processes can be understood in terms of whether and how decision-makers/negotiators succeed in modifying their own interrogative agendas or those of their counterparts, and the outcomes of these processes can be described in terms of the “common ground” agenda thus reached. 
Also, phenomena such as polarization \cite{myers1976group}, echo chambers \cite{sunstein2001republic} and self-fulfilling prophecies \cite{merton1948self} can be understood in terms of the formation and dynamics of interrogative agendas among networks of agents.

\paragraph{Logical modelling of interrogative agendas.} \label{prelim:Logical modelling} As discussed above, interrogative agendas are in essence (conjunctions of) questions. An influential approach in logic \cite{groenendijk1984studies} represents questions as equivalence relations over a suitable set of possible worlds $W$ (representing the possible states of affairs e.g.~relative to a given situation). The equivalence relations on any set form a general complete lattice $\mathrm{E}(W)$ \cite{birkhoff1940lattice}, an ordered algebra 
(which is ‘general’ in the sense that the distributivity laws $x\wedge(y\vee z = (x\wedge y)\vee(x\wedge z)$ and $x\vee(y\wedge z) = (x\vee y)\wedge (x\vee z)$ do not need to hold in it),
 which formally represents the hierarchical structure of interrogative agendas discussed above. Although the lattices $\mathrm{E}(W)$ are in general not even distributive, they resemble powerset algebras in  some important respects, for instance in their being join-generated and meet-generated by their atoms and co-atoms respectively (for a proof see the Appendix\ref{prop:charact join-irr}).  The following proposition characterizes the set $\jty(\mathrm{E}(W))$ of  atoms and the set $\mty(\mathrm{E}(W))$ of co-atoms of $\mathrm{E}(W)$\footnote{For any lattice $\mathbb{A}$, the symbols $\jty(\mathbb{A})$ and $\mty(\mathbb{A})$ are  standard  for denoting the sets of the completely join- and meet-irreducible  elements of $\mathbb{A}$, respectively \cite{gehrke2004bounded,dunn2005canonical}}.
 
 \begin{prop}
\label{prop:charact meet-irr}
For any set $W$, if $|W|\geq 2$, then 
\begin{enumerate}
    \item $e\in \mty(\mathrm{E}(W))$ iff $e$ is identified by some  partition of the form $\mathcal{E}_{X}: = \{X, W\setminus X\}$ with $\varnothing \subsetneq X \subsetneq W$.
\item $e\in \jty(\mathrm{E}(W))$ iff $e$ is identified by some partition of the form $\mathcal{E}_{xy}: = \{\{x, y\}\}\cup \{\{z\}\mid z\in W\setminus \{x, y\} \}$ with $x, y \in W$ such that $x\neq y$.
 \end{enumerate}
\end{prop}

 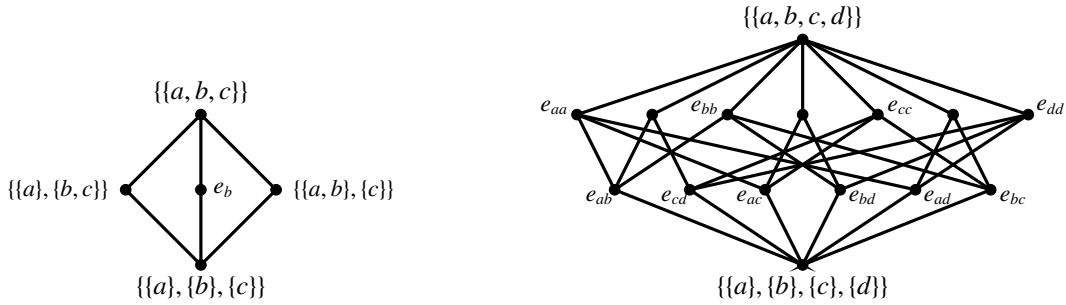
\begin{figure}[h!]
\begin{tikzpicture}
\draw[very thick] (0, -1) -- (0, 1) --
	(-1, 0) -- (0, -1) -- (1, 0) -- (0, 1);
	
	\filldraw[black] (0,-1) circle (2 pt);
	\filldraw[black] (0, 1) circle (2 pt);
	\filldraw[black] (-1, 0) circle (2 pt);
	\filldraw[black] (1, 0) circle (2 pt);
	\filldraw[black] (0, 0) circle (2 pt);
		\draw (0, 1.3) node {$\{\{a, b, c\}\}$};
	\draw (0, -1.3) node {$\{\{a\}, \{b\}, \{c\}\}$};
	\draw (-1.9, 0) node {{\small{$\{\{a\}, \{b, c\}\}$}}};
    \draw (0.3, 0) node {{\small{$e_b$}}};
    \draw (1.9, 0) node {{\small{$\{\{a, b\}, \{c\}\}$}}};


    \filldraw[black] (8,2) circle (2 pt);

    \filldraw[black] (6,1) circle (2 pt); 
\filldraw[black] (8,1) circle (2 pt); 
\filldraw[black] (10,1) circle (2 pt); 

  \filldraw[black] (5,1) circle (2 pt); 
   \draw (4.7, 1.1) node {{\small{$e_{aa}$}}};
    \filldraw[black] (7,1) circle (2 pt); 
      \draw (6.7, 1.1) node {{\small{$e_{bb}$}}};
       \filldraw[black] (9,1) circle (2 pt); 
         \draw (9.3, 1.1) node {{\small{$e_{cc}$}}};
          \filldraw[black] (11,1) circle (2 pt); 
          \draw (11.3, 1.1) node {{\small{$e_{dd}$}}};
          \draw[very thick] (8, 2) -- (5, 1);
           \draw[very thick] (8, 2) -- (7, 1);
            \draw[very thick] (8, 2) -- (9, 1);
             \draw[very thick] (8, 2) -- (11, 1);

              \draw[very thick] (5, 1) -- (5.5, 0); 
              \draw[very thick] (5, 1) -- (7.5, 0);
              \draw[very thick] (5, 1) -- (9.5, 0);
	
              \draw[very thick] (7, 1) -- (5.5, 0); 
              \draw[very thick] (7, 1) -- (8.5, 0);
              \draw[very thick] (7, 1) -- (10.5, 0);

              \draw[very thick] (9, 1) -- (6.5, 0); 
              \draw[very thick] (9, 1) -- (7.5, 0);
              \draw[very thick] (9, 1) -- (10.5, 0);

              \draw[very thick] (11, 1) -- (6.5, 0); 
              \draw[very thick] (11, 1) -- (8.5, 0);
              \draw[very thick] (11, 1) -- (9.5, 0);
\filldraw[black] (5.5,0) circle (2 pt); 
 \draw (5.3, -0.1) node {{\small{$e_{ab}$}}};
\filldraw[black] (6.5,0) circle (2 pt); 
 \draw (6.3, -0.1) node {{\small{$e_{cd}$}}};
\filldraw[black] (7.5,0) circle (2 pt); 
 \draw (7.3, -0.1) node {{\small{$e_{ac}$}}};
\filldraw[black] (8.5,0) circle (2 pt); 
 \draw (8.8, -0.1) node {{\small{$e_{bd}$}}};
\filldraw[black] (9.5,0) circle (2 pt); 
 \draw (9.8, -0.1) node {{\small{$e_{ad}$}}};
\filldraw[black] (10.5,0) circle (2 pt); 
 \draw (10.8, -0.1) node {{\small{$e_{bc}$}}};
\filldraw[black] (8,-1) circle (2 pt);
\draw[very thick] (8, 2) -- (6, 1) --
	(5.5, 0) -- (8, -1) -- (6.5, 0) -- (6, 1);	
\draw[very thick] (8, 2) -- (8, 1) --
	(7.5, 0) -- (8, -1) -- (8.5, 0) -- (8, 1);	
\draw[very thick] (8, 2) -- (10, 1) --
	(9.5, 0) -- (8, -1) -- (10.5, 0) -- (10, 1);	
    	\draw (8, 2.3) node {$\{\{a, b, c, d\}\}$};
    \draw (8, -1.3) node {$\{\{a\}, \{b\}, \{c\}, \{d\} \}$};
\end{tikzpicture}
\caption{\footnotesize{The lattices of equivalence relations on  the three-element set $W := \{a, b, c\}$,  and the four-element set $W: =\{a, b, c, d\}$. In the lattice on the left, $e_b$ corresponds to the partition $\{\{b\}, \{a, c\}\}$. In the lattice on the right,  $e_{xy} = \{\{x\}, \{y\}, W\setminus\{x, y\}\}$ for all $x, y\in \{a, b, c, d\}$, and the unlabelled nodes correspond, from left to right, to the partitions $\{\{a, b\}, \{c, d\}\}$,  $\{\{a, c\}, \{b, d\}\}$, and $\{\{a, d\}, \{b, c\}\}$, respectively.}}
\end{figure}

It is well known that every general lattice is a sublattice of the lattice of equivalence relations on some set \cite{whitman1946lattices}. This immediately implies that the negation-free fragment of classical propositional logic without the distributivity axioms (which we refer to as the basic non-distributive logic) is sound and complete w.r.t.~the class of lattices of equivalence relations. Hence, the basic non-distributive logic can be regarded as the basic logic of interrogative agendas. This basic framework naturally lends itself to be enriched with various kinds of logical operators, such as epistemic operators, which represent the way in which the interrogative agenda of an agent (or a group of agents) is perceived or known by another agent (or group), and dynamic operators, which encode the changes in agents’ interrogative agendas. Of particular interest for the sake of this research line is the possibility of enriching the basic framework with heterogeneous operators, suitable to encode the interaction among different kinds of entities; for instance, operators that associate (groups of) agents $c$ with their (common) interrogative agenda $\Diamond c$, or operators that associate pairs $(e,\phi)$, such that $e$ is an interrogative agenda and $\phi$ is a formula (event), with the formula $e{\pdla}\phi$,  representing the content of $\phi$ ‘filtered through’ the interrogative agenda $e$.  On the basis of these ideas, a fully-fledged formal epistemic theory of the interrogative agendas of social groups and individuals can be developed, and in the following sections we will start building this theory.

As discussed in the introduction, in the present paper we aim at modelling the different ways different agents categorize, based on the different subsets $Y$ of a given set $X$ of features they consider relevant.
Let $u$ be an irrelevant feature which every object under consideration has. Let $X' = X \cup \{u\}$. We identify the features which are not in the given interrogative agenda with the irrelevant feature $u$ via an equivalence relation. This can be interpreted as treating these features as irrelevant. Accordingly, we will model the epistemic stance of an agent $j$ who considers the features in $Y\subseteq X$ relevant as the interrogative agenda (i.e.~the equivalence relation on $X'$) identified by the partition $\{\{x\}\mid x\in Y\}\cup \{X'\setminus Y\}
$. That is, the interrogative agenda of agent $j$ identifies  each feature that $j$ considers relevant only with itself, and identifies all the other features with the irrelevant feature $u$.  Notice that equivalence relations corresponding to partitions of this shape are exactly those that are meet-generated by co-atoms identified by bi-partitions of the form $\{\{x\}, X'\setminus\{x\}\}$, that is, bi-partitions of $X'$ such that one cell is a singleton set with an element from $X$.   In what follows, rather than working with the whole lattice $\mathrm{E}(X')$, we will only work with the sub-meet-semilattice $\mathbb{D}$ of $\mathrm{E}(X')$ which is meet-generated by those elements $m_x\in \mty(\mathrm{E}(X'))$ which are identified by  bi-partitions of the form $\{\{x\}, X'\setminus\{x\}\}$, for each $x\in X$. The following proposition provides an equivalent representation of $\mathbb{D}$, which somewhat simplifies the exposition.

\begin{prop}
For every set $X$, the sub-meet-semilattice $\mathbb{D}$ of $\mathrm{E}(X')$ which is meet-generated by those elements $m_x\in \mty(\mathrm{E}(X'))$ identified by  bi-partitions of the form $\{\{x\}, X'\setminus\{x\}\}$  is order-isomorphic to $(\mathcal{P}(X), \supseteq)$.
\end{prop}
\begin{proof}
For every $Y\in \mathcal{P}(X)$, let $e_Y\in \mathrm{E}(X')$ be  identified by the partition $\{\{x\}\mid x\in Y\}\cup \{X'\setminus Y\}
$. As discussed above, $e_Y = \bigsqcap\{m_x\mid x\in Y\}$, where for each $x\in Y$, the co-atom $m_x\in \mty(\mathrm{E}(X'))$ is the one identified by the bi-partition  $\{\{x\}, X'\setminus\{x\}\}$. Hence, $e_Y\in \mathbb{D}$. Conversely, for each $e\in \mathbb{D}$, let $Y_e: = \{x\in X'\mid e\leq m_x\}\in \mathcal{P}(X)$. It is straightforward to verify that $e_{Y_e} = e$ and $Y_{e_{Y}} = Y$ for any $e\in \mathbb{D}$ and $Y\in \mathcal{P}(X)$. Moreover, $e_1\leq e_2$ iff $Y_{e_1}\supseteq Y_{e_2}$, as required. 
%
\end{proof} 

\paragraph{Financial statements networks.} \label{financial statements prelim} A {\em financial statements network} is constructed from journal entry data, which describe the change in financial position and that is readily available in all companies. The journal entry records show how much money flows from one set of financial accounts to another set. These entries are generated by their underlying business process, for example, the Sales process. A financial statements network is a bipartite digraph $\mathbb{G}=(A \cup X, E)$  with $X$ as the set of financial account nodes and $A$ as the set of business process nodes and $E \subseteq X \times A$ as the set of directed edges. Clearly, each bipartite digraph $\mathbb{G}=(A \cup X, E)$ as above can be equivalently represented as a formal context $\mathbb{P} = (A, X, E)$. The set of financial accounts can be obtained from the journal entry data. A business process is derived from the journal entry structure. The structure represents the relative amounts debited and credited for each financial account. Although amounts can be different, all journal entries with the same structure are considered equal. A formal definition of a {\em business process} can be written as \cite{boersma2018financial}
\begin{equation} \label{eq:business process definition}
    a: \sum_{x_i \in A }\alpha_ix_i \implies \sum_{y_j \in X} \beta_jy_j
\end{equation}
where $m$ is the number of credited financial accounts, $n$ is the number of debited financial accounts and $\alpha_i$ is the relative amount with respect to the total credited and $\beta_j$  the relative amount with respect to the total debited. The arrow here represents the flow of money between the accounts. The (weighted) edges between nodes $x_i$ or $y_j$ and $B$ are the coefficients $\alpha_i$ and $\beta_j$ from the business process definition in Equation \eqref{eq:business process definition}.

\section{Interrogative agendas, coalitions, and categorization} \label{sec:Interrogative agendas, coalitions, and categorization} 
In this section, we introduce the logical framework we use to formalize different agents and their agendas or features of interest and the categorizations obtained from such agendas. 
\paragraph{Types.} Let $\mathbb{P} = (A,X,I) $ be a  (finite) formal context obtained from a bipartite graph.  Let $C$ be a (finite) set of agents.  Let $\mathbb{C} = (\mathcal{P}(C), \cup,\cap, (\cdot)^\mathrm{c}, \bot,\top)$ and $\mathbb{D}=(\mathcal{P}(X), \sqcup,\sqcap, (\cdot)^\mathrm{c}, \epsilon, \tau)$\footnote{The  operation $(\cdot)^{\color{red}\mathrm{c}}$ denote the complement with respect to both $C$ and $X$. In general, it will be clear from the context which one of the two is used. For any set $B\subseteq C$ (resp.\ $Y\subseteq X$), $\bigcup B$ (resp.\ $\bigsqcup Y$) and $\bigcap B$ (resp.\ $\bigsqcap Y$) are the arbitrary joins (resp.\ meets)  on the  lattices.} be the Boolean algebras join-generated and meet-generated by elements of $C$  and  $X$, respectively. (The explanation for taking the lattice meet-generated by $X$ is given in \ref{prelim:Logical modelling}).
Any agent $j$ and issue $m$ can be identified with an atom (join-generator) of $\mathbb{C}$  and  a co-atom (meet-generator) of $\mathbb{D}$, respectively. That is, agents and issues in our formal model are defined as elements of the sets $C$ and $D$ respectively. 
We interpret any $c \in \mathbb{C}$ as a coalition of agents $j \in C$,  and we read $j \leq c$ as agent $j$ being a member of coalition $c$. Similarly, we interpret  any $d \in \mathbb{D}$ as the interrogative agenda supporting all the issues $m \in X$ such that $m \geq d$. That is, in this formal model coalitions and interrogative agendas are defined to be arbitrary subsets of $C$ and $\mathbb{D}$ respectively. 

From here on, we use $d \in \mathbb{D}$, and the set $Y=\{ m \in X \mid m \geq d\}$ interchangeably to denote an interrogative agenda consisting of all issues in the set $Y$.

\paragraph{Unary heterogeneous connectives. } Consider the following relation:
\[R\subseteq X\times C\quad\quad m R j\; \mbox{ iff\; issue } m \mbox{ is {\em relevant} to agent } j.\]
The relation $R$ induces the  operations $\Diamond, {\rhd}: \mathbb{C}\to \mathbb{D}$ defined as follows: for every agent $j$, let $\Diamond j = {\rhd}j: = \bigsqcap R^{-1}[j]$, where $R^{-1}[j] := \{m\mid m R j\}$. Then, for every $c\in \mathbb{C}$, 
\[\Diamond c: = \bigsqcup \{\Diamond j\mid j\leq c\} \quad\quad {\rhd}c: = \bigsqcap \{{\rhd} j\mid j\leq c\}.\]

Under the intended interpretation of $R$,
for every coalition $c$, the interrogative agenda ${\rhd} c$  denotes  the {\em distributed agenda} of $c$ (i.e.~${\rhd} c$ is the interrogative agenda supporting exactly those issues supported by {\em at least one} member of $c$), while
$\Diamond c$ is the {\em common agenda} of $c$ (i.e.~$\Diamond c$ is the interrogative agenda supporting exactly those issues supported by {\em all} members of $c$). Algebraically: \[{\rhd}c\leq m\quad \mbox{ iff }\quad {\rhd}j\leq m\; \mbox{ for some }\; j\leq c \quad\quad\quad
\Diamond c\nleq m\quad \mbox{ iff }\quad \Diamond j\nleq m\; \mbox{ for some }\; j\leq c.\]
For an example illustrating  common and distributed agendas see Section \ref{subsec:Deliberation in crisp case}. 
\begin{prop}  \label{order}
For all $c_1, c_2\in \mathbb{C}$,
\begin{enumerate}
\item $ \Diamond(c_1\cup c_2) = \Diamond c_1 \sqcup \Diamond c_2$ and $\Diamond \bot = \epsilon$;
\item ${\rhd} \bot = \tau$ and ${\rhd}$ is antitone;
\item  $ {\rhd}(c_1\cup c_2) = {\rhd} c_1 \sqcap {\rhd} c_2$,
\end{enumerate}
where $\epsilon$,  and $\tau$ denote the bottom and top of the Boolean algebra $\mathbb{D}$ respectively.
\end{prop}
\begin{proof}
1. By definition, $\Diamond \bot = \bigsqcup\{\Diamond j\mid j\leq \bot\} = \bigsqcup \varnothing = \bot_{\mathbb{D}} = \epsilon$. If $c_1\leq c_2$ then $\{\Diamond j\mid j\leq c_1\}\subseteq \{\Diamond j\mid j\leq c_2\}$, and hence $\Diamond c_1: = \bigsqcup\{\Diamond j\mid j\leq c_1\}\leq \bigsqcup\{\Diamond j\mid j\leq c_2\}: = \Diamond c_2$. This implies that $\Diamond c_1\sqcup \Diamond c_2 \leq \Diamond (c_1\cup c_2)$. Conversely, $\Diamond c_1 \sqcup \Diamond  c_2\leq m$ iff $\Diamond c_1 \leq m$ and $ \Diamond c_2\leq m$, i.e.~$\bigsqcup\{\Diamond j\mid j\leq c_1\}\leq m$ and $\bigsqcup\{\Diamond j\mid j\leq c_2\}\leq m$, iff $\Diamond j\leq m$ for every $j\leq c_1$ and every $j\leq c_2$, i.e.~$\Diamond(c_1\cup c_2): = \bigsqcup\{\Diamond j\mid j\leq c_1\cup c_2\}\leq m$.

2. By definition, ${\rhd} \bot = \bigsqcap\{{\rhd} j\mid j\leq \bot\} = \bigsqcap \varnothing = \top_{\mathbb{D}} = \tau$. If $c_1\leq c_2$ then $\{{\rhd} j\mid j\leq c_1\}\subseteq \{{\rhd} j\mid j\leq c_2\}$, and hence ${\rhd} c_2: = \bigsqcap\{\Diamond j\mid j\leq c_2\}\leq \bigsqcap\{{\rhd} j\mid j\leq c_1\}: = {\rhd} c_1$. This implies that ${\rhd} (c_1\cup c_2) \leq {\rhd} c_1\sqcap {\rhd} c_2$.

3. By 2, it is enough to show that if $m\in X$ and ${\rhd} (c_1 \cup c_2)\leq m$, then ${\rhd} c_1 \sqcap {\rhd} c_2\leq m$. Since $m$ is completely meet-prime, the assumption $\bigsqcap \{{\rhd}j\mid j\leq c_1\cup c_2\} = {\rhd} (c_1 \cup c_2)\leq m$ implies ${\rhd}j\leq m$ for some $j\leq c_1\cup c_2$. Since $j$ is completely join-prime, $j\leq c_1\cup c_2$ implies $j\leq c_1$ or $j\leq c_2$. In either case, ${\rhd} c_1\sqcap {\rhd}c_2\leq {\rhd}c_i\leq {\rhd}j \leq m$, as required.
\end{proof}
Since  $\mathbb{D}$ is a Boolean algebra, we can define $\Box c: = \neg \Diamond \neg c$. Notice that $\Diamond \neg c$ is the interrogative agenda which supports all the issues that are considered relevant by all the agents out of coalition $c$. Hence, \[\Box c: = \bigsqcap \{m \mid \exists j (j\nleq c \& \Diamond j\nleq m)\}.\] 
Consider the following relation:
\[U\subseteq X \times C \quad\quad m U j \mbox{ iff issue } m \mbox{ is {\em irrelevant } to agent } j.\]
When $\mathbb{D}$ is a Boolean algebra, $R$ and $U$ can be expressed in terms of each other,  but in general this is not the case. We might define two more diamond and right-triangle type operators in terms of $U$. That is, some issues  are positively relevant to any given agent, others which are positively irrelevant and others which are neither. In this case, for any $j\in C$, we assume that $R^{-1}[j]\cap U^{-1}[j] = \varnothing$ .
%
\[\Diamond_U c: = \bigsqcup \{\Diamond_U j\mid j\leq c\} \quad\quad {\rhd_U}c: = \bigsqcap \{{\rhd_U} j\mid j\leq c\}.\]

 Let $X_1 \subseteq X$ be an interrogative agenda. This agenda induces a context $\mathbb{P}_1=(A,X_1,I_1)$, where $I_1 = I \cap {A} \times X_1$ on $\mathbb{P}$. Formal context $\mathbb{P}_1$ denotes the context of interest for an agent with agenda $X_1$. Let $\mathcal{R}$ denote the set of all  formal contexts induced from $\mathbb{P}$.
 We define an information ordering on all such induced contexts as follows: For any $\mathbb{P}_1, \mathbb{P}_2 \in \mathcal{R}$ , 
\[
\mathbb{P}_1=(A,X_1,I_1) \leq_I \mathbb{P}_2=(A,X_2,I_2) \quad \text{iff} \quad X_1 \subseteq X_2.
\]

The order $\leq_I$ defines  a lattice on the set of induced formal contexts from $\mathbb{P}$, $\mathbb{R} = (\mathcal{R},\vee_I,\wedge_I) $ as follows.  For any $\mathbb{P}_1=(A,X_1,I \cap A \times X_1)$, and $\mathbb{P}_2=(A,X_2,I \cap A \times X_2))$, 
\[
\mathbb{P}_1 \wedge_I \mathbb{P}_2= (A,X_1 \cap X_2,I_1 \cap I_2)
\]
\[
\mathbb{P}_1 \vee_I \mathbb{P}_2= (A,X_1 \cup X_2, I \cap A \times (X_1 \cup X_2) ).
\]

\begin{prop} \label{prop:order concept preservation}
Let $\mathbb{P}_1$, $\mathbb{P}_2$ be the formal contexts induced from $\mathbb{P}$ such that  $\mathbb{P}_1 \leq_I \mathbb{P}_2$. Then for any $B \subseteq A$ if $B$  is Galois-stable in  set in $\mathbb{P}_1$, then it is Galois-stable in $\mathbb{P}_2$ as well. 
\end{prop}
\begin{proof}
It is enough to show that $ I_2^{(0)}I_2^{(1)}[B] \subseteq B$. As $\mathbb{P}_1 \leq_I \mathbb{P}_2$, we have $X_1 \subseteq X_2$, and $I_1= I_2 \cap A \times X_1$. Thus, $I_1^{(1)}[B] \subseteq I_2^{(1)}[B]$ and for $Z\subseteq X_1$, $I_1^{(0)}[Z]=I_2^{(0)}[Z]$, which together imply: $$I_2^{(0)}[I_2^{(1)}[B]]\subseteq I_2^{(0)}[I_1^{(1)}[B]]= I_1^{(0)}[I_1^{(1)}[B]].$$ By the assumption that $B$ is Galois-stable in  set in $\mathbb{P}_1$, it follows that $I_1^{(0)}[I_1^{(1)}[B]]=B$. Hence $ I_2^{(0)}I_2^{(1)}[B] \subseteq B$. This concludes the proof.
\end{proof}
Thus, a larger agenda (and hence larger context in $\leq_I$ ordering) corresponds to a larger formal context in information ordering, and in turn to a finer categorization. This is consistent with the intuition that the larger agenda means larger information considered for differentiating  objects leading to finer categorization. The following corollary is immediate from the Proposition \ref{prop:order concept preservation}.
\begin{cor}\label{cor:up-set Galois stability}
For any $B \subseteq A$, let $\mathcal{R}_B$ be the set of induced formal contexts  of $\mathbb{P}=(A,X,I) $ for which  $B$ is Galois-stable. Then, $\mathcal{R}_B$ is an up-set of $\mathcal{R}$.
\end{cor}

\begin{remark}
The set $\mathcal{R}_B$ need not be a filter in $\mathcal{R}$. Consider $\mathbb{P}=(A , X,I) $, where $A=\{a,b\}$, $X=\{x_1,x_2,x_3\}$ and $I=\{(a,x_1),(a,x_2),(a,x_3), (b,x_2) \}$. Let $X_1, X_2 \subseteq X$ be the sets $X_1=\{x_1,x_2\}$ and  $X_2=\{x_2,x_3\}$. Let $X_3=X_1 \cap X_2=\{x_2\}$. It is clear that the set $B=\{a\}$ is Galois stable in $\mathbb{P}_1 = (A, X_1, I \cap A \times X_1)$ and $\mathbb{P}_2 = (A, X_2, I \cap A \times X_2)$ , but not in $\mathbb{P}_1 \wedge_I \mathbb{P}_2 = (A, X_3, I \cap A \times X_3)$.  Thus, $\mathcal{R}_B$  is not a filter in $\mathcal{R}$.
\end{remark}

\begin{definition}
Let $\mathcal{R}$ be the set of all formal contexts induced by $\mathbb{P} =(A,X,I)$ and $\mathbb{R}$ be the set of corresponding concept lattices. This induces maps on corresponding concept lattices as well given by $f_1,f_2:\mathbb{C} \to \mathcal{R}$
\[
f_1(c)=(A, X_1, I \cap A \times X_1), \quad \text{and}
\]
\[f_2(c)=(A, X_2, I \cap A\times X_2),
\]
where $X_1= \{m \in X \mid m \geq \Diamond c\}$, and 
$X_2= \{m \in X \mid m \geq \rhd c\}$.

\end{definition}

In case we also have irrelevance issue $U$, we define maps $g_1,g_2:\mathbb{C} \to \mathcal{R}$ given by 
\[
g_1(c) = (A, X_5, I \cap A\times X_5),  \quad \text{and}
\]
\[
g_2(c) = (A, X_6, I \cap A\times X_6),
\]
where $X_5=X\setminus X_3$,  and $X_6=X\setminus X_4$, for $X_3=\{m \in X \mid m \geq \rhd_{U} c\}$ and  $X_4=\{m \in X \mid m \geq \Diamond_{U} c\}$. As we assume  $R^{-1}[j]\cap U^{-1}[j] = \varnothing$ for any $j\in C$, for any $c \in \mathbb{C}$
\[
X_1 \cap X_4 = \emptyset \quad \text{and} \quad X_2 \cap X_3= \emptyset.
\]
Therefore,
\[
X_1 \subseteq  X\setminus X_4=X_6 \quad \text{and} \quad X_2 \subseteq  X\setminus X_3=X_5.
\]
Hence,
\[f_1(c) \leq_I g_2(c) \quad \text{and} \quad f_2(c) \leq_I g_1(c).
\]

For any $c$, $f_1(c)$ and $f_2(c)$ are interpreted as the "categorization according to features of interest to the all members of  $c$" and  "categorization according to features  of interest to  at least one  member of  $c$" respectively. In similar way,
$g_1(c)$ and $g_2(c)$ are interpreted as the "categorization according to features which are not considered irrelevant by all members of  $c$" and  "categorization according to features which are not considered irrelevant by any member of  $c$" respectively.

In concrete situations, relation $U$ can be used to denote certain features which an agent explicitly mentions should not be relevant to categorization.  For example, in auditing an auditor may want to exclude features like gender or race of people involved in a transaction from  criterion used for categorization  to avoid discrimination. Under this interpretation of $U$, the contexts $g_1(c)$ and $g_2(c)$ can be viewed as  the finest categorization acceptable to at least one member of $c$", and " the finest categorization acceptable to the all members of $c$" respectively.
In case we have both relations $R$ and $U$ describing issues of interest and issues explicitly mentioned to be irrelevant by agents the set of induced contexts (or categorizations) $\mathbb{P}_1$ satisfying 
\[f_1(j)=f_2(j) \leq \mathbb{P}_1 \leq g_1(j)=g_2(j)\]
can be interpreted as set of  categorizations acceptable to $j$. 

The following proposition defines order theoretic properties of operations $f_1,f_2,g_1,g_2$ in terms of information ordering $\leq_I$ on induced formal contexts.

\begin{prop} \label{prop:lattice of categorizations}
For $c_1, c_2 \in \mathbb{C}$, and  $f_1$,  $f_2$, $g_1$,  $g_2$ as defined above, we have 
\begin{enumerate}
\item $f_1(c_1) \leq_I f_2(c_1)$.
\item $g_2(c_1) \leq_I g_1(c_1)$.
\item If $c_1 \leq c_2$, then 
\[f_1(c_2) \leq_I f_1(c_1), \quad f_2(c_1) \leq_I f_2(c_2), \quad g_1(c_1) \leq_I g_1(c_2), \quad g_2(c_2) \leq_I g_2(c_1).\]
    \item $f_1(c_1 \vee c_2)= f_1(c_1) \wedge_I f_1(c_2)$.
    \item $f_2(c_1 \vee c_2)= f_1(c_1) \vee_I f_1(c_2)$.
    \item $g_1(c_1 \vee c_2)= g_1(c_1) \vee_I g_1(c_2)$.
    \item $g_2(c_1 \vee c_2)= g_2(c_1) \wedge_I g_2(c_2)$.
\end{enumerate}
\end{prop}
\begin{proof}
Follows immediately from the definitions and Proposition \ref{order}.
\end{proof}
\subsection{ Substitution of issues in  agendas} \label{ssec:Substitution of issues in  agendas}

Consider the following relation:
\[S\subseteq X\times C\times X \quad\quad S(n, j, m)\; \mbox{ iff\;  agent $j$ would {\em substitute} issue $m$ with issue $n$. } \] 
Here by substituting an issue we mean replacing an issue in agenda with another.  The relation $S$ induces the  operation $\pdla: \mathbb{C}\times \mathbb{D}\to \mathbb{D}$ defined as follows: for every agent $j$ and issue $m$, let $j\pdla m   : = \bigsqcap S^{(1)}[j, m]$, where $S^{(1)}[j, m] := \{n\mid S(n, j, m)\}$. Intuitively, $j\pdla m$ is the interrogative agenda supporting exactly the issues that agent $j$ prefers to issue $m$. Relation $S$ can be used to model deliberation in many scenarios. Deliberation often involves agents substituting issues from other agents' agenda in attempt to reach a common agreement. This can lead us to a new compromised agenda containing issues which may not be present in the initial  agenda of any of the agents. The substitution relation allows us to model such deliberation scenarios (See example in Section \ref{subsec:Deliberation in crisp case}).

Then, for every $c\in \mathbb{C}$ and $e\in \mathbb{D}$,
\[ c\pdla e: = \bigsqcup \{j\pdla m\mid j\leq c \mbox{ and } e\leq m\}.\]
Intuitively, $c\pdla e$ is the agenda representing the shared view among the members of $c$ of how the issues in $e$ should be modified.
\begin{prop}
For every $e\in \mathbb{D}$ and all $c_1, c_2\in \mathbb{C}$,
\begin{enumerate}
\item $ \bot \pdla e = \epsilon$ and $c\pdla \tau = \epsilon$;
\item $ (c_1\cup c_2) \pdla e =  c_1 \pdla e \sqcup  c_2 \pdla e$;
\item  $\pdla$ is antitone in its second coordinate;
\item $c \pdla (e_1\sqcap e_2) =  c \pdla e_1 \sqcup  c \pdla e_2$,
\end{enumerate}
where $\epsilon$,  and $\tau$ denote the bottom and top of the Boolean algebra $\mathbb{D}$ respectively.
\end{prop}
\begin{proof}
1. By definition, $ \bot \pdla e = \bigsqcup\{ j\pdla m \mid j\leq \bot \mbox{ and } e\leq m\} = \bigsqcup \varnothing = \bot_{\mathbb{D}} = \epsilon$. Likewise, $ c \pdla \tau = \bigsqcup\{ j\pdla m \mid j\leq c \mbox{ and } \tau\leq m\} = \bigsqcup \varnothing = \bot_{\mathbb{D}} = \epsilon$.

2. If $c_1\leq c_2$ then $\{j\pdla m \mid j\leq c_1 \mbox{ and } e\leq m\}\subseteq \{j\pdla m \mid j\leq c_2 \mbox{ and } e\leq m\}$, and hence $ c_1 \pdla e: = \bigsqcup\{j\pdla m \mid j\leq c_1 \mbox{ and } e\leq m\}\leq \bigsqcup\{j\pdla m \mid j\leq c_2 \mbox{ and } e\leq m\} = c_2 \pdla e$. This implies that $c_1 \pdla e\sqcup c_2 \pdla e \leq (c_1\cup c_2)  \pdla e$. Conversely,
$c_1 \pdla e\sqcup c_2 \pdla e\leq n$ iff $c_1 \pdla e\leq n$ and $ c_2 \pdla e\leq n$, i.e.~$\bigsqcup\{j\pdla m \mid j\leq c_1 \mbox{ and } e\leq m\}\leq n$ and $\bigsqcup\{j\pdla m \mid j\leq c_2 \mbox{ and } e\leq m\}\leq n$, iff $j\pdla m\leq n$ for every $m$ and $j$ such that either  $m\geq e$ and $j\leq c_1$, or $m\geq e$  and  $j\leq c_2$. Given that any $j$ is completely join-prime, this implies that $(c_1\cup c_2) \pdla e: = \bigsqcup\{j\pdla m \mid j\leq c_1\cup c_2 \mbox{ and } e\leq m\}\leq n$.

3. If $e_1\leq e_2$ then $\{j\pdla m \mid j\leq c \mbox{ and } e_2\leq m\}\subseteq \{j\pdla m \mid j\leq c \mbox{ and } e_1\leq m\}$, and hence $ c \pdla e_2: = \bigsqcup\{j\pdla m \mid j\leq c \mbox{ and } e_2\leq m\}\leq \bigsqcup\{j\pdla m \mid j\leq c \mbox{ and } e_1\leq m\} = c \pdla e_1$. This implies that $c \pdla e_1\sqcup c \pdla e_2 \leq c  \pdla (e_1\sqcap e_2)$.

4. By 3, it is enough to show that if $n\in X$ and $c \pdla e_1\sqcup c \pdla e_2\leq n$ then $c \pdla (e_1\sqcap e_2)\leq n$. The assumption
$c \pdla e_1\sqcup c \pdla e_2\leq n$ is equivalent to $c \pdla e_1\leq n$ and $ c \pdla e_2\leq n$, i.e.~$\bigsqcup\{j\pdla m \mid j\leq c \mbox{ and } e_1\leq m\}\leq n$ and $\bigsqcup\{j\pdla m \mid j\leq c \mbox{ and } e_2\leq m\}\leq n$, iff $j\pdla m\leq n$ for every $m$ and $j$ such that either $j\leq c$ and $m\geq e_1$, or $j\leq c$ and  $m\geq e_2$. If any $m$ is completely meet-prime, this implies that $c \pdla (e_1\sqcap e_2): =
\bigsqcup\{j\pdla m \mid j\leq c \mbox{ and } (e_1\sqcap e_2)\leq m\}\leq n$.
\end{proof}
 The basic requirement for $S$ of a rational agent is the following condition of coherence with $R$.
 \begin{definition}
  We say that the relation $S$ is \em{coherent} with $R$ if
  \[
  \forall j \forall m[ m Rj \implies S(m,j,m)   ].
  \]
 \end{definition}
 The coherence condition can be interpreted as if agent $j$ considers $m$ to be relevant issue then she will be okay with replacing it with itself. 
 There can be several other conditions on $S$ which may be of the interest depending on particular scenarios. We consider study of such conditions and their representation in the  language of modal logic as a future topic of study.

\section{Deliberation and categorization} \label{sec:Deliberation and categorization- crisp case}
In this section, we consider the deliberation scenario between two agents when they have a fixed crisp set of features they relevant for categorization. Let $j_1$ and $j_2$ be two agents with agendas $X_1=R^{-1}(j_1)$ and $X_2=R^{-1}(j_2)$ respectively. We consider two natural outcomes of deliberation. The first possible outcome is to consider their common agenda i.e. the intersection of agendas of both agents. In this case,  categorization after deliberation is given by the set of features
 \[X_1 \cap X_2= \{ m \in X \mid m \geq \Diamond c\},\]
 where $c$ is the coalition of $j_1$ and $j_2$. The second possibility is to consider their distributed agenda i.e.  the intersection of agendas of both agents. In this case,  categorization after deliberation is given by the set of features
 \[X_1 \cup X_2=  \{ m \in X \mid m \geq \rhd c\}\]
 where $c$ is the coalition of $j_1$ and $j_2$. Thus, these categorizations are given by $f_1(c)$, and $f_2(c)$ respectively.  
 
 If  we also have irrelevance relation $U$ for agents, the categorizations $g_1(c)$  (resp.  $g_2(c)$) can be seen as the result of deliberation when they decide to exclude the issues considered irrelevant (or undesirable) by either of them (resp. both of them).
 
 \subsection{Substitution relation in deliberation} \label{ssec:Substi deliberation-crisp}
 Let $j_1$ and $j_2$ be agents with agendas $Y_1=R^{-1}[j_1]$  and $Y_2=R^{-1}[j_2]$. Let $S$  be the  substitution relation giving preferences of agents $j_1$ and $j_2$ in substituting issues for each other. We assume that $S$ is coherent.   We consider  the following two  possible outcomes of deliberation between $j_1$ and $j_2$. Let $e_1=\sqcap_{y \in Y_1}y $ and $e_2=\sqcap_{y \in Y_2}y $. 
 \begin{enumerate} 
    \item \textbf{Substitution-union}
    \begin{equation}  \label{substi-union}
          e=j_1 \pdla e_2 \sqcap j_2 \pdla e_1
    \end{equation}
    This result of deliberation can be interpreted as follows. The agenda $j_1 \pdla e_2$ (resp. $ j_2 \pdla e_1$) consists of all the issues $j_1$ (resp. $j_2$) considers better to substitute compared to any issue in the agenda of $j_2$ (resp. $j_1$). Thus, $j_1 \pdla e_2$, and $ j_2 \pdla e_1$ can be seen as $j_1$ and $j_2$ considering each other's agendas and using their substitution preferences (given by $S$) to propose a version of other person's agenda more agreeable to them. The $\sqcap$ operation takes the union of these substituted versions of agendas proposed by $j_1$ and $j_2$.
     \item\textbf{Substitution-intersection}
      \begin{equation}  \label{substi-intersection} 
          e'=j_1 \pdla e_2 \sqcup j_2 \pdla e_1
      \end{equation} 
      The interpretations of $j_1 \pdla e_2 $ and $j_2 \pdla e_1$ have been discussed in previous paragraph. The $\sqcup$ operation takes the intersection of these substituted versions of agendas proposed by $j_1$ and $j_2$.
 \end{enumerate}
Let $Y= \{ m \in X \mid m \geq e\}$ and $Y'= \{ m \in X \mid m \geq e'\}$. 
Then the contexts (categorizations) $\mathbb{P}_1= (A,Y, I \cap A \times Y)$ and $\mathbb{P}_1'= (A,Y', I \cap A \times Y')$ describe the categorizations as a result of deliberation between agents according to \eqref{substi-union}, and \eqref{substi-intersection} respectively. For an example  to see the effect of subtitution for crisp agendas in deliberation, see Section \ref{sssec:Example with substitution}. 
The following proposition gives some order theoretic properties of formal contexts obtained from the agendas aggregated using  \eqref{substi-union} and/or \eqref{substi-intersection}.
\begin{prop} \label{prop:substi order context crisp case}
 Let $S$ be any substitution relation. Then 
 \begin{enumerate}
     \item  $\mathbb{P}_1' \leq_I \mathbb{P}_1$.
     \item  If $S$ is coherent, then 
    \[f_1(j_1 \vee j_2) \leq_I  \mathbb{P}_1'.\]
 \end{enumerate}
\end{prop}
\begin{proof}
1. As $e \leq e'$, we have $ Y' \subseteq Y$. Therefore, by definition of $ \leq_I$, we have   $\mathbb{P}_1' \leq_I \mathbb{P}_1$.

2. As $S$ is coherent, we have $Y_1 \cap Y_2 \subseteq  \{ y \in X \mid y \geq j_1 \pdla e\}$ and $Y_1 \cap Y_2 \subseteq  \{ y \in X \mid y \geq j_2 \pdla e_2\}$. Thus, $Y_1 \cap Y_2 \subseteq Y'$. Therefore, $f_1(j_1 \vee j_2) \leq_I \mathbb{P}_1'$.
\end{proof}
Hence, the categorization obtained from \eqref{substi-intersection} is always coarser than the categorization obtained from \eqref{substi-union}. This is consistent with the intuition as in \eqref{substi-intersection}, and  \eqref{substi-union} to obtain aggregated agenda agents decide to take intersection, and union of the substituted agendas respectively. Moreover, if $S$ is coherent the issues of interest to  both agents are also part of their aggregated agenda and hence obtained categorization is coarser than categorization given by their common agenda. This also  provides justification for coherence being a rationality condition. Indeed if a feature is considered relevant by both agents in deliberation, it is natural that it should be considered relevant by them after deliberation.
 \section{Non-crisp interrogative agendas} \label{sec:Non-crisp interrogative agendas}
 In this section, we extend ideas developed so far  to the non-crisp case.
Suppose different agents have mass functions  describing their interest or preference in different set of issues i.e. for agent $j$ we have a mass function $m_j: \mathcal{P}(X) \to [0,1]$ where for any $Y \subseteq X$, $m_j(Y)$ denotes preference of agent to use set of features (agenda) $Y$ as a criterion for categorization. This mass function can be seen as a Dempster-Shafer mass function on the set $\mathcal{P}(X)$. Two particular cases of interest from a practical point of view are when $m_j$ is simple or consonant. 
 
\begin{definition}
For any Dempster-Shafer mass function  $m:\mathcal{P}(S) \to [0,1]$ a set $Y \subseteq S$ is said to be \em {focal set} of $m$ if $m(Y)>0$. 
\end{definition}
 \begin{definition}
 A Dempster-Shafer mass function  $m:\mathcal{P}(S) \to [0,1]$ is said to be
 \begin{enumerate}
     \item {\em simple} iff it has at most one focal set apart from $S$.
    \item{  \em consonant} iff the set of focal sets of $m$  form a chain.
 \end{enumerate}
 \end{definition}
Suppose that an agent $j$ mentions that she considers a set of features $Y \subseteq X$ to be of high importance (given by $\alpha \in [0,1]$) for categorization. In this case, the agenda of agent $j$ may be represented by a simple mass function $m_j(Y)=\alpha$ and $m_j(X)=1-\alpha$. Here, we do not have any information about how $j$ intends to distribute importance (or preference) $1-\alpha$ between different features for categorization. Thus, we assign this mass to the set $X$ to denote non-availability of information regarding its distribution.  Another situation where a simple mass function may arise in deliberation scenario is when different agents may be given different importance. For example in auditing, suppose we have two agents $j_1$ and $j_2$ with agendas $X_1$ and $X_2$. Suppose the relative importance (or influence or trust) of $j_1$ and $j_2$ in an organization (informally understood) are given by $\alpha$ and $1-\alpha$ (normalized). In such situations, their agendas can be effectively  represented by by mass functions $m_{j_1}(X_1)= \alpha$, $m_{j_1}(X)= 1-\alpha$, and 
 $m_{j_2}(X_2)= 1-\alpha$, $m_{j_2}(X)= \alpha$.

 In some situations, an agent may give a  list of increasing sets of features  describing extent to which these sets of features  are important for categorization. For example,  for some ${Y}_1 \subseteq {Y}_2 \subseteq {Y}_3 \subseteq X$, and $0 \leq  \alpha_1 \leq \alpha_2 \leq \alpha_3\leq 1$, an agent may say that "If we consider all features in $Y_1$, ${Y}_2$ and ${Y}_3$ for categorization this should describe a good (or required) categorization  to extent $\alpha_1$, $\alpha_2$, and $\alpha_3$ respectively." In this case, the agenda of the agent can be represented by a consonant mass function $m$ with 
 \[ m({Y}_1)=\alpha_1 \quad m({Y}_2)=\alpha_2-\alpha_1 \quad  m({Y}_3)=\alpha_3-\alpha_2 \quad m(X)=1-\alpha_3.\]
  \paragraph{Irrelevant issues in non-crisp case.} We can use Dempster-Shafer mass function to denote issues which an agent may consider irrelevant as follows.
 
 \begin{example}
 Suppose an agents $j$ considers set of issues $Y \subseteq X$ to be irrelevant for categorization. Suppose (normalized to $1$) trust/importance of agent $j$ is given by $\alpha_j \in [0,1]$. Then this information is represented by $\overline{m_j}: \mathcal{P}(X) \to  [0,1]$ given by 
 \[
 \overline{m_j}(Y^\mathrm{c})=\alpha_j \quad \text{and}  \quad  \overline{m_j}(X)=1-\alpha_j.
 \]
 In case, we have both interest and irrelevance information represented by mass functions as discussed above their aggregated agenda is given by combining all such mass functions.
 \end{example}
\begin{example}
Suppose an agent $j$ assigns different importance to each feature individually. In this case the mass function representing agenda of $j$ is given  by a simple mass function $m_j: \mathcal{P}(X) \to  [0,1]$, for any $Y \subseteq X$, 
\[
m_j(Y) =\sum_{y \in Y}v(x),
\]
where $v(x)$ is the  (normalized) importance value assigned by agent $j$ to $x$. 
\end{example}

Any Dempster-Shafer  mass function $m:\mathcal{P}(X) \to [0,1]$, induces a probability mass function $m': \mathcal{R} \to [0,1]$ given by
 \[m'((A, {Y}, I \cap A\times {Y} ))= m(Y).\]
 For any $\mathbb{P} \in \mathcal{R}$, $m'(\mathbb{P})$ gives extent to which categorization $\mathbb{P}$ is preferred by agent with agenda given by $m$. Thus, given a non-crisp agenda represented by a Dempster-Shafer  mass function on $\mathcal{P}(X)$, we obtain a preference function (which can also be seen as a probability function) on the contexts (or categorizations) induced  from $\mathbb{P}$.  For any non-crisp agenda $m$, induced probability  mass function $m'$ on $\mathcal{R}$ defines a probability function $p_{m'}$ on $\mathcal{R}$ as follows. For any $\mathcal{V} \subseteq \mathcal{R}$,  
\[
p_{m'}(\mathcal{V}) =\sum_{ \mathbb{P} \in \mathcal{V}}m'( \mathbb{P}).
\]

\subsection{Non-crisp agendas in decision-making}\label{ssec:Non-crisp agendas in decision-making}
We have seen that the non-crisp agendas can be represented by Dempster-Shafer mass function $m:\mathcal{P}(X) \to [0,1]$. Such a mass function induces a   probability or preference function over the set of induced categorizations $\mathcal{R}$. This function assigns a value to categorization showing its relevance/preference of the probability of it being desired categorization. However, once such non-crisp categorization is obtained, we need to use this in decision-making task at the hand. In some situations, all the different categorizations and their probability/preference values  may be assessed by an expert. However, this may not be feasible in the most practical applications due to large data sizes and lack of  assessment tools. Another natural way to use non-crisp agendas in decision-making is to obtain a crisp categorization approximating this non-crisp categorization.
Here, we discuss some possible ways to obtain such approximations. The most natural choice is to consider the categorization with the highest preference or probability value attached to it. However, this choice ignores a large amount of information of interest in other alternative categorizations. Here, we propose a novel  {\em stability-based method}\footnote{The concept of stability in formal concept analysis  was introduced by Kuznetsov in \cite{kuznetsov2007stability}. The stability measure was introduced to estimate stability of concept in a crisp formal context with respect to changes in features. Here, we define stability index for non-crisp concepts instead to estimate their Galois-stability.} to form a crisp formal context from given probability function on the set of induced contexts (categorizations).
 
 \subsubsection{Stability-based method}\label{sssec:Stability-based method}
 
 \begin{definition}
 Let $\mathbb{P} =(A, X,I)$ be the formal context under consideration and let $\mathcal{R}$ be the set of induced formal contexts. Let $m':\mathcal{R} \to [0,1]$ be the induced probability mass function on categorization induced by an agenda given by mass function $m:\mathcal{P}(X) \to [0,1]$. Then for any $G \subseteq A$, the stability index of $G$ is given by\footnote{A set $G$ is Galois-stable in $\mathbb{P}_1$ if $G = \mathsf{Cl}^{\mathbb{P}_1}_1(G) = G$, where $\mathsf{Cl}^{\mathbb{P}_1}_1$ denotes the extension closure of $\mathbb{P}_1$. }
 \[
 \rho_m(G) =\sum \{m'(\mathbb{P}_1) \mid G \, \text{is Galois-stable in}\, \mathbb{P}_1 \}.
 \] 
 \end{definition}
 For any $G$, $ \rho_m(G)$ denotes the likelihood of $G$ being a Galois-stable set under non-crisp agenda $m$. In formal concept analysis Galois-stability is interpreted as stability of a concept showing its tendency to form a meaningful category or concept definable both in terms of its intensions and extensions. Thus, for any $G \subseteq A$, its stability index $ \rho_m(G)$ denotes the tendency or probability of $G$ forming a meaningful and stable category (or a concept).

 For any $\beta \in [0,1]$, we define a $\beta$-categorization on $A$ as follows:
Let  
\[
\mathbb{P}(m,\beta)=\{\mathsf{Cl}(G) \mid G \subseteq A, \rho_m(G) \geq \beta\}.
\]
For any $G$ with $ \rho_m(G) \geq \beta$, the set $\mathcal{R}_G$ contains the set of formal contexts in which $G$ is Galois-stable.   Therefore,   $\mathbb{P}(m,\beta)$ denotes the categorization consisting of all the closed sets in (i.e. sets of form $\mathsf{Cl}(G)$) with stability index larger than $\beta$. The set $\mathbb{P}(m,\beta)$ under set-theoretic inclusion forms a lattice which can be used to depict the categorization  $\mathbb{P}(m, \beta)$. This lattice can be interpreted as the concept lattice corresponding to the given (non-crisp) agenda $m$ and stability parameter $\beta$. The categorization $\mathbb{P}(m,\beta)$ obtained in above manner can provide a good representation of categorization preferences given by a non-crisp agenda given by $m$. Unlike choosing the categorization with the highest probability categorization, this method takes into account opinions or information about other possible categorizations as well. The stability parameter $\beta$ allows us to choose our required stability threshold for a concept or category to be relevant and can be used to regulate size of obtained categorizations.  Now, we prove some basic properties of the categorizations obtained by this method. For any $B \subseteq A$ and $X_1 \subseteq X$ we use  $B \in \mathbb{P}_1=(A,X_1, I \cap A \times X_1)$ to denote $B$ is a Galois-stable set in the context $\mathbb{P}_1$. 

\begin{prop}
 Let $m: \mathcal{P}(X) \to [0,1] $ be a mass function representing an agenda. Let $\beta_1, \beta_2 \in [0,1]$ be such that $\beta_1 \leq \beta_2$. Then for any $B \subseteq A$, we have $B \in \mathbb{P}(m,\beta_2)$ implies ${B} \in \mathbb{P}(m,\beta_1)$.
\end{prop}
 \begin{proof}
 Let ${B} \subseteq A$ be such that ${B} \in  \mathbb{P}(m,\beta_2)$. Then 
 $B=\mathsf{Cl}(G)$ for some $G \subseteq A$ such that  $\rho_m (G) \geq \beta_2$. As $\rho_m (G) \geq \beta_2$,  we have $\rho_m (G) \geq \beta_1$. Thus, $\mathsf{Cl}(G)=B \in \mathbb{P}(m,\beta_1)$. 
 \end{proof}
  This result matches with the intuition that the lower value of $\beta$ means that our stability index threshold for considering a category or concept is lower and hence gives a finer categorization. 
  
  \begin{remark}
By Proposition \ref{prop:order concept preservation}, if a set $G \subseteq A$ is Galois-stable in an induced context $\mathbb{P}_1=(A,X_1, I \cap A \times X_1)$ for some $X_1 \subseteq X$, then it is  Galois-stable in the context $\mathbb{P} =(A,X,I)$. Thus, a set which is not Galois-stable in $\mathbb{P}$ is not Galois-table in any induced context and has stability index $0$. Thus, for any $\beta>0$, we only need to check the Galois-stable sets in $\mathbb{P}$ to find sets with stability index greater than $\beta$ needed to obtain categorization $\mathbb{P}(m,\beta)$ for any mass function $m$. In fact, we only need to check  set $G \subseteq A$ which are Galois-stable in a context  $\mathbb{P}_1=(A,X_1, I \cap A \times X_1)$ for some $X_1 \subseteq X$ with $m(X_1)>0$ to find the sets with stability index greater than $\beta$.
  \end{remark}
 
 There have been several orderings defined on Dempster-Shafer mass functions \cite{denoeux2006cautious, dubois1986set, smets2002application, yager1986entailment}. We mention some of them in the following definition.
 \begin{definition}[\cite{klawonn1992dynamic,denoeux2006cautious}] \label{def:mass-ordering}
 For any $m_1, m_2: \mathcal{P}(X) \to [0,1]$, 
 \begin{enumerate}
     \item \emph{pl-ordering}:  $m_1 \leq_{\mathrm{pl}} m_2$  iff for every ${Y} \subseteq X$, $pl_1({Y}) \leq pl_2({Y})$.
     \item  \emph{q-ordering}:  $m_1 \leq_{\mathrm{q}} m_2$ iff for every ${Y} \subseteq X$, $q_1({Y}) \leq q_2({Y})$.
     \item  \emph{s-ordering}:  $m_1 \leq_{\mathrm{s}} m_2$ iff there exists a square matrix $S$ with general term $S(W,{Y}),W,Y \subseteq  X $ verifying
     \[
     (\forall Y\subseteq  X) \quad \sum_{{W} \subseteq X} S(W,{Y}) =1,
     \]
     \[
      (\forall W, Y \subseteq X) \quad S(W,{Y}) > 0 \implies W \subseteq {Y},
     \]
     such that 
     \[
     (\forall W \subseteq X) \quad m_1(W)= \sum_{Y \subseteq X}S(W,Y)m_2({Y}).
     \]
     \item \em{Dempsterian specialization ordering}:  $m_1 \leq_{\mathrm{d}} m_2$ iff
     there exists a Dempster-Shafer mass function $m$ such that $m_1= m \cap m_2$. Where, $m_1 \cap m_2$ denotes the un-normalized Dempster's combination given by
     \begin{equation} \label{eq: un-normalized DS rule}
          m_1 \cap m_2({Y})= \sum_{{Y}_1 \cap {Y}_2={Y}}m_1({Y}_1)m_2({Y}_2).
     \end{equation}
 \end{enumerate}
  \end{definition}

 It is well known that \cite{denoeux2006cautious}
 \begin{equation} \label{eq:mass-ordering implications}
   m_1 \leq_d m_2 \implies m_1 \leq_s m_2 \implies \begin{cases} m_1 \leq_{\mathrm{pl}} m_2 \\
  m_1 \leq_q m_2.
  \end{cases}
 \end{equation}
 \begin{remark}
 If un-normalized Dempster's combination rule is replaced with  Dempster's combination rule in the definition of order $\leq_d$, the first implication in  \ref{eq:mass-ordering implications} does not hold in general. The required  counter-example is given as follows.
 Let $S= \{s_1,s_2\}$ and $m_1(\{s_1\})=1$, $m_2(\{s_1\})=0.1$, 
 $m_2(\{s_2\})=0.9$. Let $m_1=m \oplus m_2$. Then we have $m_1(\{s_1\})=1$. It is clear that $m_1=m \oplus m_2$, but $m_1 \nleq_s m_2$.
 \end{remark}
 
 We define a new order on  Dempster-Shafer  mass functions as follows.
 
 \begin{definition}
 For any $m_1,m_2: \mathcal{P}(S) \to [0,1]$ we define {\em up-set restricted order} $\leq_\uparrow$ as follows. $m_1 \leq _\uparrow m_2$ iff for any up-set  (i.e.~any upward closed subset) $\mathcal{V} \in \mathcal{P}\mathcal{P}(S) $, 
\[ \sum_{Y \in \mathcal{V}} m_1 (Y)  \leq \sum_{Y \in \mathcal{V}} m_2 (Y) \]
 \end{definition}
 \begin{prop}
  \begin{equation} \label{eq: new mass-ordering implications}
   m_1 \leq_d m_2 \implies m_1 \leq_s m_2\implies m_1 \leq _\uparrow m_2 \implies \begin{cases} m_1 \leq_{\mathrm{pl}} m_2 \\
  m_1 \leq_q m_2.
  \end{cases}
 \end{equation}
 \end{prop}
 \begin{proof}
 By property \eqref{eq: new mass-ordering implications} we only need to prove the implications involving $\leq _\uparrow$. 
 
 1. \[ m_1 \leq_s m_2\implies m_1 \leq _\uparrow m_2\]
 Suppose $ m_1 \leq_s m_2$. Then by defintion of $\leq_s$ there exists a square matrix $S$ with general term $S(W,{Y}),W,Y \subseteq  X $ verifying
  \[
      S(W,{Y}) \geq 0 \implies W \subseteq {Y}, \quad  \forall W,Y \subseteq X,
     \]
     such that 
     \[
     m_1(W)= \sum_{Y \subseteq X}S(W,Y)m_2({Y}), \quad  \forall W \subseteq X.
     \]
Notice preliminarily that, by the above, $m_1(W)=\sum_{W\subseteq Y}S(W,Y)m_2({Y})$.
Let $\mathcal{V} \in \mathcal{P}\mathcal{P}(S) $ be an up-set. Then $$\sum_{W\in\mathcal{V}}m_1(W)=\sum_{W\in\mathcal{V}} \sum_{W\subseteq Y}S(W,Y)m_2({Y})=\sum_{Y\in\mathcal{V}}\sum_{W\subseteq Y\ \&\ W\in\mathcal{V}}S(W,Y)m_2(Y)\leq\sum_{Y\in\mathcal{V}}m_2(Y),$$
the last inequality following from the fact that $\sum_{{W} \subseteq X} S(W,{Y}) =1$.  That is, $m_1 \leq_\uparrow m_2$.

2. \begin{equation*} 
 m_1 \leq _\uparrow m_2 \implies \begin{cases} m_1 \leq_{\mathrm{pl}} m_2 \\
  m_1 \leq_q m_2.
  \end{cases}
 \end{equation*}
It follows immediately from the fact that sets $\{Z \subseteq  X \mid Y \subseteq Z\}$,  and $\{Z \subseteq  X \mid Y \cap  Z \neq \emptyset \}$ are up-sets. 
 \end{proof}
  \begin{lemma} \label{lem: rho-order}
Suppose $m_1$, $m_2$ are agendas such that $m_1 \leq_\uparrow m_2$. Then, for any $G\subseteq A$, $\rho_{m_1}(G) \leq \rho_{m_2}(G)$. 
 \end{lemma}
 As $m_1 \leq_\uparrow m_2$, we have $m_1' \leq_\uparrow m_2'$, where $m_1'$ and $m_2'$ are mass functions induced on $\mathcal{R}$ by $m_1$ and $m_2$ respectively.
 Let $\mathcal{U} \subseteq \mathcal{R}$ be any  up-set in $ \mathcal{P} (\mathcal{R})$.  Then 
 \[ \sum_{\mathbb{P} \in \mathcal{U}} m_1' (\mathbb{P} )  \leq \sum_{\mathbb{P} \in \mathcal{U}} m_2' (\mathbb{P} ). \]
    That is, for any up-set $\mathcal{U}$ in $\mathcal{P}$,
    $p_1' (\mathcal{U}) \leq p_2'(\mathcal{U})$. Let $\mathcal{R}_G$ be the set of formal contexts in which $G$ is Galois-stable.
    As  $\mathcal{R}_G$ is an up-set by Corollary \ref{cor:up-set Galois stability}, we have 
    \[
       \rho_{m_1}(G)= p_{m_1'}(\mathcal{R}_G) \leq p_{m_2'}(\mathcal{R}_G)= \rho_{m_2}({G} ).
    \]

 \begin{prop}\label{prop:specification-categorization} 
  Let $m_1,m_2: \mathcal{P}(X) \to [0,1] $ be the mass functions defining two agendas. If $m_1 \leq_\uparrow m_2$, then for any  fixed $\beta \in [0,1]$, and ${B} \subseteq A$, if ${B} \in \mathbb{P}(m_1,\beta)$, then ${B} \in \mathbb{P}(m_2,\beta)$.
 \end{prop}
 \begin{proof}
 Let $B \subseteq A$ be such that  ${B} \in  \mathbb{P}(m_1,\beta)$. Then,   there exists $G \subseteq A$ such that $\rho_{m_1} (G) \geq \beta$ and $\mathsf{Cl}(G)=B$. By Lemma \ref{lem: rho-order}, we have $\rho_{m_2} (G) \geq \beta$. Thus, $\mathsf{Cl}(G)=B \in \mathbb{P}(m_2,\beta)$. 
 \end{proof}
Therefore, if $m_1 \leq _\uparrow m_2$, for any fixed stability parameter $\beta \in [0,1]$, the categorization obtained from $m_1$ by the stability-based method  is coarser than the one obtained from $m_2$. As a smaller mass function in $ \leq _\uparrow $  order corresponds to a more specific agenda i.e. less amount of information being considered in categorization, it is reasonable that this gives a coarser categorization than a larger mass function in $ \leq _\uparrow $ order.
The following Corollary is an immediate implication of Proposition \ref{prop:specification-categorization}, and property \eqref{eq: new mass-ordering implications}.
 \begin{cor} \label{cor:Dempster order}
 Suppose $m_1$, $m_2$ are agendas such that $m_1 \leq_s m_2$ (or $m_1 \leq_d m_2$). Then, for any $G \subseteq A$, $\rho_{m_1}(G) \leq \rho_{m_2}(G)$. Moreover, for any ${B} \subseteq A$, ${B} \in \mathbb{P}(m_1,\beta)$ implies  ${B} \in \mathbb{P}(m_2,\beta)$.
 \end{cor}
 
 \subsubsection{Methods via transformation to probability} \label{sssec:Methods via transformation to probability} 
 Let $m:\mathcal{P}(X) \to [0,1]$ be the mass function describing an agenda. We can use methods in Dempster-Shafer theory  to transform mass functions into probability functions to estimate importance of each feature in categorization (see Section \ref{ssec:Importance of different features in these categorizations} for an example). Two well known methods for transforming a mass function to a probability function are  plausibility transform \cite{cobb2006plausibility} and pignistic transformation \cite{klawonn1992dynamic, smets2005decision}.
 \begin{definition}
Let  $m_1:\mathcal{P}({S}) \to [0,1]$ be any mass function. 
\begin{enumerate}
    \item The {\em pignistic transformation} of $m$ is $\mathrm{bet}_p:\mathcal{P}({S}) \to [0,1]$\footnote{The notation $\mathrm{bet}$ comes from Smets introduced in \cite{smets1990constructing}. The 'bet' stands for bets as motivation behind transformation. The word pignistic comes from Latin word 'pignus' meaning bets.} given by the following. For any $s \in {S}$,
    \[
    \mathrm{bet}_p(m)(s)= \sum_{s \in {Y}} \frac{m( {Y}) }{ |{Y}|}.
    \]
    \item  The {\em plausibility transformation} of $m$ is $ \mathrm{Pl}\_P:\mathcal{P}({S}) \to [0,1]$ given by the following. For any $s \in {S}$,
    \[
    \mathrm{Pl\_P}(m)(s)= \frac{\pl(s)}{\sum_{s \in {S}}\pl(s)}.
    \]
\end{enumerate}
 \end{definition}
 Both $\mathrm{bet}_p(m)$ and $\mathrm{Pl\_P}(m)$ can be used to estimate the importance of different features in categorizations. Thus, this approach allows us to estimate  importance of individual features in agendas (and categorization given by them) obtained possibly from complex deliberation process. These estimated importance values also provide an alternative method for categorization when we are only interested in flat categorization i.e. partition of objects. The importance  values of features  can be used as weights in computing proximity or dissimilarity between different objects based on features shared and not shared between them. The  dissimilarity or proximity data obtained in such a way can be used to cluster objects based on several machine learning techniques \cite{jain1999data,jain1988algorithms}. For more detailed properties of these transformations and comparative study refer to \cite{cobb2003comparison}.  These methods can only provide a flat categorization or clustering, unlike stability-based method which provides a hierarchical categorization.

 \section{Deliberation and categorization in the non-crisp case}\label{sec:Deliberation and categorization-non crisp case}
 In this section, we try to model deliberation scenarios when agendas of agents are given by Dempster-Shafer mass functions as discussed in previous section.  Dempster-Shafer theory has been used for aggregating preferences of different agents \cite{sentz2002combination,frittella2020toward}. In previous section we have discussed that the mass functions describing non-crisp agenda can be interpreted as a priority or preference function for agents describing the priority assigned by an agent to a categorization. Thus, we believe that Dempster-Shafer based preference aggregation is a reasonable way to model aggregation of these agendas through deliberation.
 
 Let us consider two agents $j_1$ and $j_2$ with their agendas given by mass functions $m_1$ and $m_2$ on $\mathcal{P}(X)$. 
 
 \textbf{Common agenda --} Given two mass functions $m_1$ and $m_2$ representing agendas of two agents $j_1$ and $j_2$, their common agenda is given by their Dempster-Shafer combimation \cite{shafer1976mathematical} $m_1 \oplus m_2$  given as follows. For any ${Y} \subseteq X$, $Y \neq \emptyset$, 
 \begin{equation}\label{eq:DS-rule}
      m_1 \oplus m_2 ({Y}) = \frac{\sum_{{Y}_1 \cap {Y}_2= {Y}} m_1({Y}_1)m_2({Y}_2)}{\sum_{{Y}_1 \cap {Y}_2 \neq \emptyset}m_1({Y}_1)m_2({Y}_2)}.
 \end{equation}
 and $m_1 \oplus m_2 (\emptyset) =0$. 
 
 \textbf{Distributed agenda --} Given two mass functions $m_1$ and $m_2$ representing agendas of two agents $j_1$ and $j_2$, their distributed agenda is given by mass function $m_1 \cup_m m_2$ given as follows. For any ${Y} \subseteq X$,
 \begin{equation}\label{eq:inverse DS-rule} 
   m_1 \cup m_2 ({Y}) = \sum_{{Y}_1 \cup_m {Y}_2= {Y}} m_1({Y}_1)m_2(Y_2).
 \end{equation}
 $m_1 \cup_m m_2$ is indeed a  mass function as $\sum_{Y \subseteq X} =\sum_{Y \subseteq X} \sum_{{Y}_1 \cup {Y}_2= {Y}} m_1({Y}_1)m_2(Y_2) = \sum_{{Y}_1} \subseteq X m_1(Y_1) \sum_{{Y}_2} \subseteq X m_2(Y_2)  =1$. The mass functions $m_1 \oplus m_2 $ and $m_1 \cup_m m_2$ can be considered to represent the common (normalized to ignore  conflicts) and distributed agendas of $j_1$ and $j_2$ respectively. Justification for this interpretation is as follows. If the agendas of agents $j_1$ and $j_2$ are given by $Y_1$ and $Y_2$, then their common (resp. distributed) agenda is given by $Y_1 \cap Y_2$  (resp.  $Y_1 \cup Y_2$). If we assume agents $j_1$ and $j_2$ are independent the value $m_1({Y}_1)m_2(Y_2)$ can be considered as preference (or evidence) for $Y_1$ being the agenda of $j_1$ and $Y_2$ of $j_2$ simultaneously. Thus, we attach  the mass $m_1({Y}_1)m_2(Y_2)$  to  $Y_1 \cap Y_2$   (resp. $Y_1 \cup Y_2$) in case of taking common (resp.  distributed) agenda. The normalization in \eqref{eq:DS-rule} 
 allows us to ignore the completely contradictory agendas (i.e. agendas with no intersection), thus giving more weight to issues which have consensus of the agents. An example of a basic deliberation scenario involving non-crisp agendas is shown in Section \ref{ssec:Deliberation in non-crisp case}. 
 \begin{remark}
 Note that in the case of crisp agendas i.e. in case $m_1$ and $m_2$ have only one focal element $m_1 \oplus m_2$ and $m_1 \cup_m m_2$ also have only one focal element,  the corresponding categorizations with mass $1$ are $f_1(j_1 \vee j_2)$ and $f_2(j_1 \vee j_2)$ respectively. 
 \end{remark}
 
  \begin{lemma} \label{lem: combination-ordering relation} 
 For any mass functions $m_1,m_2: \mathcal{P}(S) \to [0,1]$, we have 
 \begin{enumerate}
     \item $m_1 \cap m_2 \leq_s m_1, m_2$.
 \item   $  m_1, m_2 \leq_\uparrow m_1 \cup_m m_2$.

 \end{enumerate}
 \begin{proof}
1. We show that $m_1 \cap m_2 \leq_s m_1$. The proof for $m_2$ is similar.
The proof follows by setting
\[
S(W,Y)= \sum \{m_2(Y') \mid Y' \cap Y = W  \}.
\]
It is straightforward to check that $S$ satisfies all the required conditions in Definition \ref{def:mass-ordering}.

2. For any $Y \subseteq X$, the mass $m_1(Y)$ is transferred completely to the sets larger than or equal to $Y$ in performing operation $\cup$. Thus, any mass attached by $m_1$ to any set in an up-set $\mathcal{V}$ remains in $\mathcal{V}$ in $m_1 \cup_m m_2$. In particular, for every $Y\in\mathcal{V}$ and every $Y'\subseteq X$, $Y\cup Y'\in\mathcal{V}$. Therefore, 
 \[\sum_{Y \in \mathcal{V}} m_1 (Y) =\sum_{Y \in \mathcal{V}}(\sum_{Y'\subseteq X}m_2(Y'))m_1(Y)  \leq \sum_{Y \in \mathcal{V}}\sum_{Y_1\cup Y_2=Y}m_1(Y_1)m_2(Y_2)= \sum_{Y \in \mathcal{V}} (m_1 \cup_m m_2) (Y).\]
The proof for $m_2$ is obtained in identical manner. 
 \end{proof}
 \end{lemma}

 The Proposition \ref{prop:specification-categorization}, and Corollary \ref{cor:Dempster order} can be considered as generalization of the Proposition \ref{prop:order concept preservation} to the non-crisp case. 
 The following corollary follows  from   Lemma \ref{lem: combination-ordering relation}, property \eqref{eq: new mass-ordering implications},  and Proposition \ref{prop:specification-categorization} immediately.
 \begin{cor} \label{cor:intersection-union masses}
 Suppose $m_1$ and $m_2$ are any two mass functions representing two non-crisp agendas. Then for any $\beta \in [0,1]$, and any $B \subseteq A$, 
 \begin{enumerate}
     \item  If $B \in \mathbb{P}(m_1 \cap m_2,\beta)$, then ${B} \in  \mathbb{P}(m_1,\beta)$, and   ${B} \in \mathbb{P}(m_2,\beta)$.
     \item  If   ${B} \in \mathbb{P}(m_1,\beta)$  or    ${B} \in \mathbb{P}(m_2,\beta)$, then $B \in \mathbb{P}(m_1 \cup_m m_2,\beta)$.
 \end{enumerate}
 \end{cor}

 Thus, for any non-crisp agendas $m_1$ and $m_2$ applying operation $\cap$ (resp. $\cup$) leads to a categorization coarser (resp. finer) than either of the categorizations $m_1$ or $m_2$ for any fixed stability parameter $\beta \in [0, 1]$. Hence,  Corollary \ref{cor:intersection-union masses}  can be seen  as a generalization of  the Proposition \ref{prop:lattice of categorizations} to the non-crisp case. This is again consistent with the idea that operations $\cap$ and $\cup$ give non-crisp versions of common and distributed agenda respectively. 
 
  \begin{remark} 
 The Lemma \ref{lem: rho-order}, and Proposition \ref{prop:specification-categorization} do not extend to q-ordering or pl-ordering. Let $X=\{y_1, y_2,y_3\}$, and mass functions $m_1$ and $m_2$ be as follows.
 \[
 m_1(\{y_1,y_3\})=0.3, \quad m_1(\{y_2,y_3\})=0.3, \quad m_1(\{y_1,y_2,y_3\})=0.2, \quad m_1(\{y_3\})=0.2 \quad \text{and}
 \]
 \[
 m_2(\{y_1,y_3\})=0.1, \quad m_1(\{y_2,y_3\})=0.1, \quad m_1(\{y_1,y_2,y_3\})=0.5, \quad m_1(\{y_3\})=0.3.
 \]
 It is clear that $m_1 \leq _q m_2$ and $m_1 \leq _{pl} m_2$. However, for the up-set $\mathcal{V} =\{\{y_1,y_3\}, \{y_2,y_3\}, \{y_1,y_2,y_3\}\}$ in $\mathcal{P}(X)$, we have
  \[\sum_{Y \in \mathcal{V}} m_2 (Y)  < \sum_{Y \in \mathcal{V}} m_1  (Y).\] 
 \end{remark}
 \subsection{Substitution relation in deliberation in the non-crisp case} \label{ssec:Substi deliberation non-crisp}
 
 Let $S \subseteq X \times A \times X$ be a  substitution relation giving substitution preference between different issues for involved agents. Let $m_1$ and $m_2$ be the mass functions representing agendas of $j_1$ and $j_2$ respectively.
 Here, we consider the following two possible outcomes of deliberation in this situation. For any $Z\subseteq X$, let $e_Z=\sqcap_{y \in Z}y $.
 \begin{equation} \label{eq:substi-disjunctive}
     m_1 \cup_S m_2(Y) = \sum_{Y_1 \cup Y_2=Y}m_1(Z_1)m_2(Z_2),
 \end{equation}
 For any $Y \subseteq X$, $Y \neq \emptyset$
 \begin{equation} \label{eq:substi-conjunctive}
     m_1 \oplus_S m_2(Y) =\frac{ \sum_{Y_1 \cap Y_2=Y}m_1(Z_1)m_2(Z_2)}{\sum_{Y_1 \cap Y_2\neq \emptyset}m_1(Z_1)m_2(Z_2)},
 \end{equation}
 where $Y_1= \{x \in Y \mid x \geq j_2\pdla e_{Z_1} \}$, and $Y_2= \{x \in Y \mid x \geq j_1 \pdla e_{Z_2} \}$ and $ m_1 \oplus_S m_2(\emptyset) = 0$. The denominator term $m_1(Z_1)m_2(Z_2){\sum_{Y_1 \cap Y_2\neq \emptyset}m_1(Z_1)m_2(Z_2)}$ is a normalization term and hence a well-defined mass function. The proofs that these mass functions are well-defined is similar to the proofs of the fact that combinations given by $\oplus$, $\cup$ and $\cap$ are well-defined. The interpretations of $Y_1$ and $Y_2$ were described in the discussion following \eqref{substi-union}. These combination rules can be seen as counterparts of \eqref{eq:DS-rule}, and \eqref{eq:inverse DS-rule} where the agendas of agents are replaced by 
 agendas after substitution carried out by other agent i.e. mass $m_1(Z_1)m_2(Z_2)$ (and possibly normalized) is assigned to  $Y_1 \cup Y_2$ (resp. $Y_1 \cap Y_2$) instead of $Z_1 \cup Z_2$ (resp. $Z_1 \cap Z_2$). The normalization in \eqref{eq:DS-rule} allows us to ignore the agendas which give  contradictions after substitutions (i.e. agendas $Z_1$, $Z_2$ such that that $Y_1 \cap Y_2 =\emptyset$ for $Y_1$ and $Y_2$ as defined above), thus giving more weight to the features which have consensus of the agents. An un-normalized version of \eqref{eq:substi-conjunctive} is given by 
 \begin{equation} \label{eq: unnormalized substi-conjunctive}
     m_1 \cap_S m_2 = \sum_{Y_1 \cap Y_2=Y}m_1(Z_1)m_2(Z_2).
 \end{equation}
 For an illustrative example of substitution of agendas in non-crisp case, see Section \ref{ssec:Example with substitution non-crisp}.
 \begin{remark}
 In case of crisp agendas, the aggregation rules \eqref{eq:substi-disjunctive} and \eqref{eq: unnormalized substi-conjunctive}  reduce to the aggregations given by substitution-union and substitution-intersection in Section \ref{ssec:Substi deliberation-crisp} respectively. 
 \end{remark}
 
  \begin{lemma}\label{lem:substitution aggregation rule mass ordering}
 Let $m_1, m_2: \mathcal{P}(X) \to [0,1]$ be the mass functions representing two agendas. Let $S$ be any  substitution preference relation. Then

     \[ m_1 \cap_S m_2 \leq_\uparrow m_1 \cup_S m_2.\]

 \end{lemma}
 \begin{proof}
For any  $Y_1, Y_2 \subseteq X$, $Y_1 \cap Y_2 \subseteq Y_1 \cup Y_2$. Therefore, for any up-set $\mathcal{V}$, if $Y_1\cap Y_2\in\mathcal{V}$ then $Y_1\cup Y_2\in\mathcal{V}$. Since, 
 \[
 m_1 \cap_S m_2 (Y) = \sum_{Y_1 \cap Y_2 =Y} m_1(Z_1)m_2(Z_2)
 \]
 and 
 \[ m_1 \cup_S m_2(Y) = \sum_{Y_1 \cup Y_2=Y}m_1(Z_1)m_2(Z_2),\]
 it immediately follows that  
 \[\sum_{Y \in \mathcal{V}} (m_1 \cap_S m_2)(Y)  \leq \sum_{Y \in \mathcal{V}} (m_1 \cup_S m_2) (Y)\]
as any mass attached to a set $Y$ in $m_1 \cap_S m_2$ is attached to a larger set (in inclusion order) in $m_1 \cup_S m_2$.
 \end{proof}
 The following Corollary follows immediately from Lemma \ref{lem:substitution aggregation rule mass ordering}, and Proposition \ref{prop:specification-categorization}.
 \begin{cor} \label{cor:substi formal context ordering non-crisp}
 Let $m_1$ and $m_2$ be the mass functions representing two agendas  Let $S$ be any  substitution  relation. Then for any $\beta \in [0,1]$, 
\[ B \in \mathbb{P}(m_1 \cap_S m_2,\beta), \quad \text{implies} \quad B \in \mathbb{P}(m_1 \cup_S m_2,\beta).\]
 \end{cor}
 \begin{remark}
 The Corollary \ref{cor:substi formal context ordering non-crisp} can be seen as a generalization of Proposition \ref{prop:substi order context crisp case} to the non-crisp case.
 \end{remark}
 The Corollary \ref{cor:substi formal context ordering non-crisp} says that the categorization given by $\cup_S$  is finer than the one given by $\cap_S$. This is expected as the operations $\cup_S$, and $\cap_S$ give non-crisp versions of taking union and intersection of substituted agendas respectively. Hence the agenda given by $m_1 \cup_S m_2$  considers more information than the one given by  $m_1 \cap_S m_2$ and hence gives finer categorization.

 We have described all the deliberation scenarios so far  as occurring between two agents. However, same combination methods can apply to coalitions by combining their common or distributed agenda (depending on choice of the agents) similar to the agendas of individual agents. In such cases we may have additional possibilities of taking common or distributed agendas of a coalition as their agenda. Ignoring this increase in possibilities, the deliberation between coalitions can be treated in a manner similar to the deliberation between individual agents

 \section{Examples}\label{sec:Examples end}
  We consider a small financial statements network consisting of 12 business processes  given by $a_1-a_{12}$ and 6 financial accounts  $x_1$-tax, $x_2$- revenue, $x_3$-cost of sales, $x_4$-personnel expenses, $x_5$-inventory, $x_6$-other expenses. For the details of each business process refer to database in appendix \ref{sec:dataset}. As financial statements network is a weighted-bipartite graph it can be viewed  as a many-valued formal contexts with business processes as objects and financial accounts as features. To see many-valued formal context obtained from the database in the Table \ref{database table} refer to the Table \ref{many valued context table}.
  We can use different types of conceptual scaling methods to obtain a single valued formal context from a many-valued context. Here, we use scaling with interval scaling which divides interval $[-1,1]$ into $s$ equal intervals (these are the possible weights for any edges of the financial statements network). Several other conceptual scaling methods may be appropriate depending on particular application. We do not go into details of scaling methods in this paper as it is not focus of our work. We refer to \cite{ganter1989conceptual} for more details on different conceptual scaling methods, their properties and applications.

 \subsection{Deliberation in crisp case} \label{subsec:Deliberation in crisp case}
 Let us recall the example from Section \ref{sec:Example} with a set of business processes $A=\{a_{i} \mid 1\leq i \leq 12\}$ and financial accounts $X_0= \{x_i \mid 1 \leq i \leq 6\}$. The database describing these business process is given in Table \ref{database table} and the corresponding many-valued context is given in \ref{many valued context table}. As mentioned in the introduction we use conceptual scaling to convert this many-valued context into a crisp formal context. Here,
 we use interval scaling dividing the range $[-1,1]$ into $s$ equal intervals for every attribute (i.e. financial account). Thus, the set of features for resultant crisp formal context is given by $X= \{x_{ik}\mid 1 \leq i \leq 6, 1 \leq k \leq s\}$, where a business process $a$ has feature $x_{ik}$ iff $ -1+ 2(k-1)/s \leq I(a,x_{i})  \leq -1+ 2k/s $. The categorization(concept lattice) corresponding to this formal context when $s=5$ is shown in the Figure \ref{fig:lattice 10}. This is the categorization considering to all features (financial accounts) in the context and thus contains all the information given by the context (after scaling). Therefore, this is the finest categorization obtainable by any of our proposed methods from this context. It can indeed be seen that this categorization is finer than any categorization we come across in all the categorization examples arising from different agendas (crisp or non-crisp) for this context.

 Let $C= \{j_1,j_2,j_3\}$. Interests of the agents in $C$ can be represented by the following relation $R \subseteq X \times C$. 
  \[
  R= \{(j_1,x_{1k}), (j_1,x_{2k}), (j_1,x_{5k}),(j_2,x_{1k}), (j_2,x_{2k}), (j_2,x_{3k}),  (j_3,x_{1k}), (j_3,x_{3k})                                                           \},
  \]
  where $ 1 \leq k \leq s$. The categorizations (concept lattices) obtained from the agendas of $j_1$, $j_2$, and $j_3$ are shown in the Figures  \ref{fig:lattice 1}, \ref{fig:lattice 3}, and \ref{fig:lattice 5} respectively.

  Suppose agents $j_1$, $j_2$ and $j_3$ deliberate on how to combine their agendas to get an aggregated agenda for categorization. We consider the following  two possible deliberation scenarios.

  (i) \textbf{Agents may decide to agree on their common agenda.} In this case, the aggregated agenda (after deliberation) is given by $x_{ik}$ where $i=1$ and $1 \leq k \leq s$. The resultant categorization in this case is given by $ \mathbb{P}_1= f_1(j_1 \vee j_2 \vee j_3)$. It is clear to see that this resultant categorization (context) $ \mathbb{P}_1$ is coarser than each of the categorizations $f_2(j_1)$,  $f_2(j_2)$, and  $f_2(j_3)$. Thus, this categorization takes into account only the features or information considered relevant by every  agent in deliberation. The categorization (concept lattice) obtained from this agenda is shown in the  Figure \ref{fig:lattice 7}.
  
  (ii)\textbf{ Agents may decide to agree on their distributed agenda.} In this case, the aggregated agenda (after deliberation) is given by $x_{ik}$ where $i=1,2,3,5$ and $1 \leq k \leq s$. The resultant categorization is this case is given by $ \mathbb{P}_1= f_2(j_1 \vee j_2 \vee j_3)$. It is clear to see that this resultant categorization (context) $ \mathbb{P}_1$ is finer than each of the categorizations $f_2(j_1)$,  $f_2(j_2)$, and  $f_2(j_3)$. Thus, this categorization takes into consideration all the information considered relevant  by any of the agents involved in deliberation. The categorization (concept lattice) obtained from this agenda is shown in the Figure \ref{fig:lattice 2}.

  \subsubsection{Example with substitution} \label{sssec:Example with substitution}
  We consider the following scenario of deliberation between $j_1$ and $j_2$. Suppose we have the following  substitution relation $S \subseteq X \times C \times X$, giving preferences of agents in substituting one issue with another 
  \begin{align*}
       S &=\{(x_{1k},j_1,x_{1k} ), (x_{2k},j_1,x_{2k} ), (x_{5k},j_1,x_{5k} ), (x_{3k},j_1,x_{6k} )  (x_{3k},j_2,x_{3k} ),(x_{1k},j_2,x_{1k} ),\\ 
       &(x_{2k},j_2,x_{2k} ), (x_{5k},j_2,x_{6k} ), (x_{5k},j_2,x_{5k} ) \},
  \end{align*}
  where $1 \leq k \leq s$.
In this case, if the  deliberation occurs  according to  substitution-union rule \eqref{substi-union}, the resultant  agenda  after deliberation is  $Y_1= \{x_{1k}, x_{2k},x_{5k}, x_{6k} \}$, where $1 \leq k \leq s$. Notice that the features $x_{6k}$, $1 \leq k \leq s$, (other expenses) were not present in the agenda of either of two agents, however it is present in the aggregated agenda. This corresponds to agents deliberating to choose feature which is not initial preference for either of them but both can compromise on. Preference-substitution relation allows us to model such situations in deliberation. The categorization (concept lattice) obtained from this agenda is shown in the Figure \ref{fig:lattice 9}.

In this case, if the  deliberation occurs  according to  substitution-intersection rule \eqref{substi-intersection}, the resultant  agenda  after deliberation is  $Y_2= \{x_{1k}, x_{2k}, x_{6k} \}$, where $1 \leq k \leq s$. The categorization (concept lattice) obtained from this agenda is shown in Figure \ref{fig:lattice 7}. It is clear from the concept lattices in the Figures \ref{fig:lattice 7} and \ref{fig:lattice 9} that the categorization obtained from \eqref{substi-union} is coarser than the categorization obtained from \eqref{substi-intersection}, as implied by the Proposition \ref{prop:substi order context crisp case}.
Note that the business process $a_7$ is distinguished completely from other processes with features in $Y_1$ but not in $Y_2$. Thus, if this is categorization is used to find abnormalities (considering very small categories as abnormalities) $a_7$ is likely to be flagged in the categorization with the agenda  $Y_1$ but not with the agenda  $Y_2$. In a similar way, if this categorization is used for choosing diverse sample for further processing, the process $a_7$ has much more likelihood of being chosen when agenda is $Y_1$, than $Y_2$.

 \subsection{Deliberation in non-crisp case} \label{ssec:Deliberation in non-crisp case}
 Let $j_1$, $j_2$, and $j_3$ be three agents with different possibly non-crisp agendas. Suppose the agenda of agent $j_1$ is given by mass function $m_1$ 
 \[m_1(\{x_{1k}\}) =0.6, \quad   m_1(X)=0.4,\]
 
 for $1 \leq k \leq s$. This can be considered $j_1$ assigning preference or importance  $0.6$ to the feature tax  and $0.4$ to the set of all the features i.e. preference for categorization based on tax alone describes categorization intended by $j_1$ to extent  $0.6$, while categorization based on all features in $X$ describes intended categorization fully. 
 Suppose the agenda of agent $j_2$ is given by mass function $m_2$ 
 \[m_2(\{x_{1k}\}) =0.5, \quad   m_2(\{x_{1k},x_{2k}\} )=0.3, \quad m_2(X)=0.2, \]
 for $1\leq k \leq s$.
 This can be considered as $j_2$ saying tax alone describes the categorization intended by $j_2$ to extent $0.5$, tax and revenue together describe the categorization intended by $j_2$ to extent $0.8$ and categorization based on all features in $X$ describes intended categorization fully. 
 Suppose agent $j_3$ has crisp agenda  $Y= \{x_{1k}, x_{2k}, x_{6k} \}$, $1\leq k \leq s$, and the relative importance of these agents involved in deliberation is $1:1:0.9$. In this case, the agendas of $j_1$, $j_2$, and $j_3$ adjusted for the relative importance are given by $m_1$, $m_2$, and $m_3$ respectively, where $m_3$ is given by 
 \[m_3(Y)=0.9, \quad \text{and} \quad m_3(X)=0.1.\]

The most likely categorization (i.e. categorization with the highest induced mass) for $m_1$ and $m_2$ is same and is shown in the  Figure \ref{fig:lattice 6}, while the most likely categorization according to $m_3$ is shown in the  Figure \ref{fig:lattice 7}.
 The categorizations by stability-based method for the agendas given by mass functions $m_1$, $m_2$, and $m_3$,  when $\beta=0.5$ and $s=5$  are shown in the Figures \ref{fig:lattice 6}, \ref{fig:lattice 4}, and \ref{fig:lattice 7} respectively.
 
 We again consider two possible outcomes of deliberations resulting in the aggregated agendas $m=m_1 \oplus m_2 \oplus m_3$ and $m'= m_1 \cup m_2 \cup  m_3$. The  mass functions $m$ and  $m'$ are given by 
 \[
 m(\{x_{1k}\})= 0.8, \quad   m(\{x_{1k},x_{2k}\} )=0.12, \quad  m(\{x_{1k},x_{2k},x_{6k}\} )=0.072, \quad m(X)=0.008.
 \]
 \[
  m'(\{x_{1k},x_{2k},x_{6k}\} )=0.432, \quad m'(X)=0.568.
 \]
The most likely categorizations (categorization with the highest mass)
for  the agendas $m$, and $m'$ are shown in the Figures \ref{fig:lattice 6}, and \ref{fig:lattice 10} respectively. The categorization  by stability-based method for the  agendas  $m$ and $m'$, when  $\beta=0.5$ and $s=5$ are shown in the figures \ref{fig:lattice 6}, and \ref{fig:lattice 10} respectively. As there is no conflict between $m_1$, $m_2$, and $m_3$, we have $m_1 \oplus m_2\oplus m_3=m_1 \cap m_2 \cap m_3$. Thus, as implied by the Corollary \ref{cor:intersection-union masses}, the categorization given by $m$ (resp.~$m'$) is coarser (resp. finer) than the any of the  categorizations given by $m_1$, $m_2$, or $m_3$. This can indeed be seen in the concept lattices shown.

 \subsection{Example with substitution}\label{ssec:Example with substitution non-crisp}
 We consider the following scenario of deliberation between $j_1$ and $j_2$. Suppose we have a substitution relation $S \subseteq X \times C \times X$, giving preferences of agents in substituting one issue with another. Let
  \begin{align*}
       S &=\{(x_{1k},j_1,x_{1k} ), (x_{2k},j_1,x_{6k} ),  
       (x_{1k},j_2,x_{1k} ), (x_{2k},j_2,x_{2k} ),   (x_{1k},j_2,x_{6k} ), (x_{6k},j_2,x_{6k} )\},
  \end{align*}
  where $1 \leq k \leq s$. In this case, the agendas  resulting frm the  deliberation according to \eqref{eq:substi-disjunctive} and \eqref{eq:substi-conjunctive} are $m_s=m_1 \oplus_s m_2$  and  $m_s'=m_1 \cup_s m_2$  given by 
  \[
 m_s(\{x_{1k},x_{6k}\})=0.6, \quad   m_s(\{x_{1k},x_{2k},x_{6k}\})=0.4
  \]
  \[
  m_s'(\{x_{1k}\})=0.5, \quad  m_s'(\{x_{1k},x_{6k}\})=0.5,
  \]
where $1 \leq k \leq s$. For the agenda given by $m_{s}$ there is no unique most likely categorization. The most likely categorization (categorization with the highest mass)for the agenda given by $m_s$ is shown in the Figure \ref{fig:lattice 8}. 
The categorization obtained by stability-based method for the agendas  $m_s$ and  $m_s'$, when  $\beta=0.5$ and $s=5$ is same and is shown in the Figure \ref{fig:lattice 8}. Note than for the mass functions $m_1$, and $m_2$ as above, we have $m_1 \oplus_s m_2= m_1 \cap_s m_2$. Even though, we obtain the same categorization when $\beta=0.5$, for the value of $\beta \in (0,0.4]$ the categorization obtained from  $m_s'$ would be finer than the categorization obtained from $m_s$ by the stability-based method as implied by the Corollary \ref{cor:substi formal context ordering non-crisp}.

 \subsection{Importance of different features in these categorizations} \label{ssec:Importance of different features in these categorizations}
 In this section, we give the estimated values of importance of different features (financial accounts) via pignistic and pluasibility transformations for all the the non-crisp agendas (mass functions) mentioned in different examples throughout this section. 
 \begin{table}[h]
     \centering
     \begin{tabular}{|c|c|c|c|c|c|c|}
     \hline
        Agenda & $x_1$&$x_2$& $x_3$&$x_4$&$x_5$&$x_6$\\
    \hline
          $m_1$ & 0.67 & 0.067&0.067&0.067&0.067&0.067\\
    \hline 
          $m_2$ & 0.683 & 0.183&0.033&0.033&0.033&0.033\\
    \hline 
          $m_3$ & 0.317 & 0.317&0.017&0.017&0.017&0.317\\
    \hline 
          $m$ & 0.885 & 0.085&0.001&0.001&0.001&0.025\\
    \hline 
          $m'$ & 0.239 & 0.239&0.095&0.095&0.095&0.239\\
    \hline 
          $m_s$ & 0.433 & 0.133&0&0&0&0.433\\
    \hline 
         $m_s'$ & 0.75 & 0&0&0&0&0.25\\
    \hline 
     \end{tabular}
     \caption{Importance estimates via pignistic transformation}
     \label{tab:pignistic}
 \end{table} 
  
 \begin{table}[h]
     \centering
     \begin{tabular}{|c|c|c|c|c|c|c|}
     \hline
        Agenda & $x_1$&$x_2$& $x_3$&$x_4$&$x_5$&$x_6$\\
    \hline
          $m_1$ & 0.333 & 0.133&0.133&0.133&0.133&0.133\\
    \hline 
          $m_2$ & 0.435 & 0.217 &0.087&0.087&0.087&0.087\\
    \hline 
          $m_3$ & 0.303 & 0.303&0.030&0.030&0.030&0.303\\
    \hline 
          $m$ & 0.767 & 0.153&0.006&0.006&0.006&0.061\\
    \hline 
          $m'$ & 0.213 & 0.213&0.121&0.121&0.121&0.213\\
    \hline 
          $m_s$ & 0.417 & 0.167&0&0&0&0.417\\
    \hline 
         $m_s'$ & 0.667 & 0&0&0&0&0.333\\
    \hline 
     \end{tabular}
     \caption{Importance estimates via plausibility transformation}
     \label{tab:plauisbility}
 \end{table} 
   It is clear from the tables that, by both estimation methods, the estimated importance of features differs significantly with the rules used for deliberation. For example, importance of revenue in agenda $m$ obtained by taking the common agenda of $j_1$, $j_2$, and $j_3$ is significantly higher than importance in  agenda $m'$ obtained by taking the distributed agenda of $j_1$, $j_2$, and $j_3$. These different importance values would result in significant changes in clustering obtained using these values in proximity or dissimilarity based methods as discussed in \ref{sssec:Methods via transformation to probability}. For example, business processes $a_4$ and $a_5$ have share $0.75$ and $0.95$ of revenue respectively. They are much more likely to be put into different clusters in a clustering obtained from $m$ than a clustering obtained from $m'$. These changes can have significant impacts on the performance of these methods in a given clustering task. The method proposed in this paper  uses Dempster-Shafer theory and transformations to probability functions to get estimated importance values and uses them in clustering tasks can be useful in  many applications where the different experts assign different importance to different features and might have to deliberate with each other. 
  
 \section{Conclusion and further directions} \label{sec:Conclusion and further directions}
 \paragraph{Main contributions.} The contributions of this paper are motivated by the problem of categorizing business processes for auditing purposes, in a way that facilitates the identification of anomalies. Building on the insight that different ways of categorizing might lead to widely different results, in this paper, we investigate the space of possible categorizations of business processes, seen as nodes of one type  in a bipartite graph. Formally, we regard (possibly weighted) bipartite graphs as (possibly many-valued) formal contexts.  This interpretation provides us with a way to obtain hierarchical and explainable categorization of the nodes in the bipartite graph. The structure of formal contexts allows us to have much more control over the features used for categorization. Thus, we explore the space of the possible categorizations of a given set of business processes in terms of the interrogative agendas of a given set of agents. We thus obtain categorizations useful to multiple agents with different features of interest (agendas). We use notions from modal logic to represent the interaction between different agents, agendas, and categorizations. We make some  observations about the interaction between these concepts. We then go on to discuss  possible scenarios involving deliberation of different agents in deciding relevant features and make  observations about possible outcomes (i.e.~categorizations obtained from deliberation).
 
 We generalize these results to a setting where the edges of a bipartite graph can be weighted, and so is the extent to which given agents consider given issues relevant, by using Dempster-Shafer mass functions to denote the many-valued (non-crisp) agendas of different agents involved. We also discuss  methods for obtaining a crisp (i.e.~two-valued) categorization from a given many-valued categorization, namely, we discuss the {\em stability-based method} and the {\em probability transformation  based methods}, and make some observations regarding  these methods.  We also  generalize different possible deliberation scenarios to the non-crisp case using Dempster-Shafer aggregation rules. Finally, we discuss some examples applying these ideas to the problem of categorizing business processes  from a financial statements network when different agents may have different financial accounts of interest (i.e.~different agendas) in both a crisp and a no-crisp setting.
 
 This paper initiates a new line  of research combining modal logic for  describing and reasoning about agendas and interactions, formal concept analysis for modelling explainable categorizations, and Dempster-Sha theory, for representing and computing uncertainties and many-valued priorities in categorizing  nodes in bipartite graphs regarded as formal contexts. We believe that this contribution can be applied to a much wider range of problems than those involving financial transactions, and that it lays the groundwork of a framework for modelling categorizations in business organizations or societal institutions involving  many different agents with different interest  interacting and categorizing a set of objects. Below, we discuss  some  directions for future research.
 
 \paragraph{Applying the present framework to the design categorization algorithms for auditing tasks.}
 As discussed in the Introduction and Examples sections, the  methods developed in this paper can be used to obtain explainable categorizations of business processes given by different agendas. The ensuing categorizations  are explainable, which facilitates their use and assessment  by expert auditors. These categorizations may be used to flag out possibly abnormal business processes for further checking. For example, business processes forming small categories may be seen as abnormal as they are categorized separately from  other business processes. Another way to recognize possibly abnormal business processes is to recognize abnormal members of given categories. This may be achieved via clustering business processes inside a category based on shares of different financial accounts in a business process. As discussed before, which business processes are flagged out may depend strongly on the features  and their relative importance in categorization. Thus, business processes which are abnormal across different categorizations may have a higher likelihood of being inconsistent (or at least uncommon). In future work, we intend to apply these methods to  categorization tasks for detecting abnormal business processes and other applications in auditing, and we intend to asses these methods in comparison to current techniques.  
 
 \paragraph{Modelling evidence collection from different sources relating to a company in an audit.}
 In this paper, we have used interrogative agendas to model the features of interest for different auditors or auditing sub-tasks. We study the categorizations based on these agendas,  possible deliberation scenarios be different auditors/audit methods or sub-tasks, and the categorizations resulting out of these deliberations. The same mathematical model can be flipped by using it to model the scenario of an audit team receiving  data from an organization with different departments or different records having different areas of focus. These different entities may have very different categorizations of business processes, financial transactions or other data entries based on their agendas. In such case, an audit team needs to collect and  aggregate data from these different sources having possibly different categorization of data. The logical framework developed in this paper can be used to model such situations and formalizing it. Such a formalization can be useful in  aggregation and assessment of gathered evidence from different sources in a company.

 \paragraph{Further applications in different tasks and fields. 
 }
 As discussed above, in this work we focus on modelling the interaction between different agents with different agendas and categorizations based on these agendas. 
 We believe that these formal framework can be 
 integrated in areas of research in e.g.~data-mining, information retrieval, attribute exploration,  knowledge management where formal concept analysis is already being successfully applied \cite{priss2006formal, qadi2010formal, poelmans2010formal, valtchev2004formal, poelmans2013formal, ganter2012formal, wille1996formal}.  One common problem in many applications of formal concept analysis is to reduce the size of concept lattices. The problem is due to the propensity of concept lattices to blow up in size with large data sets. Several approaches have been used in past to solve this problem \cite{cole1999scalability, singh2017concepts, dias2010reducing}. As a future direction, we intend to combine these methods with our work on categorizations based on different agendas in crisp and non-crisp cases and use these in different applications in auditing but also in other fields involving categorization 
 and social interactions like linguistics, politics and markets. 
 \paragraph{Extending to uncertain or incomplete formal contexts.}
 There have been several extensions of formal contexts dealing with situations involving incomplete or uncertain information \cite{belohlavek1999fuzzy,frittella2020toward, conradie2021rough, yao2017interval, li2013incomplete}. In our previous work, \cite{frittella2020toward}, we have used Dempster-Shafer theory to talk about evidence/belief about objects or features belonging to a category and preferences for different categories. In this work, we have focused more on uncertainty or preference functions on different categorizations of a set of objects based on different crisp or non-crisp agendas. It would be interesting to look more into using these frameworks and interaction between them. In real situations, both the problems of choosing relevant categories and assigning objects or features to a category may involve uncertainty or incomplete information and may interact with each other. It would be interesting to elaborate more on these interactions using formal frameworks developed in these papers. 
 
 \paragraph{Different modal axioms to model interaction between agents and their interrogative agendas.}
 In many practical situations, interactions between $R$, $U$ and $S$ in deliberation may satisfy additional conditions. Many of these additional conditions can be axiomatised using modal logic. We may also add additional relations and modal operators to model influences between different agents or dependencies between issues. These additional operators may satisfy some interactions axioms between themselves and the ones already considered here. Another topic of interest is to determine whether some property in this structure is modally definable or not. This leads us to investigate some general characterization of modal definability, in the style of e.g.~the Goldblatt-Thomason theorem, for such logics. Some advances in this direction have been made already \cite{conradie2018goldblatt, goldblatt2018canonical}. 
 \paragraph{Considering different combination rules and methods of interpreting non-crisp agendas.}
 In this work we have  used some basic rules from Dempster-Shafer and propose some natural methods to model aggregation of different agendas involved in deliberation. How to define different combination rules and their advantages and disadvantages has been a large topic of research in Dempster-Shafer theory \cite{sentz2002combination}. In future work, we would like to investigate which different rules may be of interest in different deliberation scenarios. 
 In this work, we have proposed stability index to combine different categorizations with different preferences into one. It would be interesting to look into other ways of aggregating these categorizations. 
 
  \bibliography{ref}
\bibliographystyle{plain}

\section*{Declaration of interest and disclaimer:} The authors report no conflicts of interest, and declare that they have no relevant or material financial interests related to the research in this paper. The authors alone are responsible for the content and writing of the paper, and the views expressed here are their personal views and do not necessarily reflect the position of their employer.

 \section{Appendix}
 \begin{prop}
\label{prop:charact join-irr}
For any set $W$, if $|W|\geq 2$, then  $e\in \jty(\mathrm{E}(W))$ iff $e$ is identified by some partition of the form $\mathcal{E}_{xy}: = \{\{x, y\}\}\cup \{\{z\}\mid z\in W\setminus \{x, y\} \}$ with $x, y \in W$ such that $x\neq y$.
\end{prop}
\begin{proof}
By construction, any $e$ corresponding to some $\mathcal{E}_{xy}$ as above is an atom, and hence a completely join-irreducible element of $\mathrm{E}(W)$. Conversely,  let $e$ be a completely join-irreducible element of $\mathrm{E}(W)$. Then $e\neq \epsilon$, hence some $x, y\in W$ exist such that $x\neq y$ and $(x, y)\in e$. To show that $e$ is identified by the partition $\mathcal{E}_{xy}: = \{\{x, y\}\}\cup \{\{z\}\mid z\in W\setminus \{x, y\} \}$, we need to show that (a) the $e$-equivalence class $X$ of $x$ cannot contain  three pairwise distinct elements, and (b) for any $z\in W\setminus \{x, y\}$, the $e$-equivalence class $Z$ of $z$ is a singleton. As to (a), assume for contradiction that $X$ contains  three pairwise distinct elements $x, y, z$. Hence, $X$ is both the union of two disjoint subsets $X_1$ and $X_2$ such that $x\in X_1$ and $y, z\in X_2$ and is the union of two disjoint subsets $Y_1$ and $Y_2$ such that $y\in Y_1$ and $x, z\in Y_2$. Then  let $e_1, e_2\in \mathrm{E}(W)$ be respectively identified by the  partitions $\mathcal{E}_1: = \{\{X_1, X_2\}\cup \{[w]_e\mid w\in W\setminus X\}$ and $\mathcal{E}_2: = \{Y_1, Y_2\}\cup \{[w]_e\mid w\in W\setminus X\}$. By construction,  $e = e_1\sqcup e_2$; however, $e\neq e_1$ and $e\neq e_2$, which contradicts the assumption that $e$ is completely join-irreducible. As to (b), by (a) and the assumptions, $X$ and $Z$ are disjoint and $X$ contains two distinct elements. If $Z$ contains some $z'\in W$ such that $z\neq z'$,  then $Z$ is  the union of two disjoint subsets $Z_1$ and $Z_2$ such that $z\in Z_1$ and $z'\in Z_2$. Then  let $e_1, e_2\in \mathrm{E}(W)$ be respectively identified by the  partitions $\mathcal{E}_1: = \{\{Z_1, Z_2\}\cup \{[w]_e\mid w\in W\setminus Z\}$ and $\mathcal{E}_2: = \{\{x\}, \{y\}\}\cup \{[w]_e\mid w\in W\setminus X\}$. By construction,  $e = e_1\sqcup e_2$; however, $e\neq e_1$ and $e\neq e_2$, which contradicts the assumption that $e$ is completely join-irreducible.
\end{proof}

 \subsection{A \ Financial statements network example} \label{sec:dataset}
 The following table is a small  database showing different transactions and financial accounts used to obtain a small financial statements network considered in examples.

 \begin{longtable}{|c|c|l|c|} 
 \caption{A small database with 12 business processes and 6 financial accounts. We use same TID to denote all credit and debit activities relating to a single business process.}\label{database table}\\
 \hline
 \textbf{ID} & \textbf{TID} & \textbf{FA name} &  \textbf{Value}\\
 \hline
 1 & 1& revenue & -100\\
 \hline
 2& 1& cost of sales& +100\\
 \hline
 3 & 2& revenue & -400\\
 \hline
 4 & 2& personal expenses  & +400\\
 \hline
5 & 3& other expenses  & +125\\
\hline
 6 & 3& cost of sales  & +375\\
 \hline
 7 & 3& revenue & -500\\
 \hline
 8 & 4& tax  & -125\\
 \hline
 9 & 4& cost of sales  & +500\\
 \hline
 10 & 4& revenue & -375\\
 \hline
  11 & 5& tax  & -10\\
  \hline
 12 & 5& cost of sales  & +200\\
 \hline
 13 & 5& revenue & -190\\
 \hline
 14 & 6& other expenses  & +50\\
 \hline
 15 & 6& cost of sales  & +450\\
 \hline
 16 & 6& inventory & -500\\
 \hline
 17 & 7 & cost of sales  & +400\\
 \hline
 18 & 7 & revenue  & -300\\
 \hline
 19 & 7 & inventory  & -100\\
 \hline
 20 & 8& revenue & -150\\
 \hline
 21& 8& cost of sales & +150\\
 \hline
 22 & 9& revenue & -250\\
 \hline
 23& 9& cost of sales &+250\\
 \hline
  24 & 10& tax & -250\\
  \hline
 24& 10& personal expenses & +250\\
 \hline
  26 & 11& revenue & -250\\
  \hline
 27& 11& personal expenses & +175\\
 \hline
 28& 11& other  expenses & +75\\
 \hline
  29 & 12& revenue & -250\\
  \hline
  30 & 12& tax & -50\\
  \hline
 31& 12& personal expenses  &+150\\
 \hline
 32& 12& other  expenses  &+150\\
 \hline
 \end{longtable}
The following table gives shows the many valued context obtained by interpreting the business processes as objects and financial accounts as features. The value of incidence relation denotes the share of given financial account in given business process.

\begin{longtable}{|c|l|c|}
   \caption{The formal context obtained from transaction database given in Table \ref{database table}. $I(a,x)=0$ for any $(a,x)$ pair not present in the table.} \label{many valued context table}\\
 \hline 
 \textbf{Business process} \textbf{(a)}& \textbf{Financial account}\textbf{ (x)}& \textbf{Share of value} \textbf{(I(a,x))}\\
  \hline
   1 ($a_1$)& revenue ($x_2$)  & -1\\
   \hline
  1 ($a_1$)& cost of sales ($x_3$)& +1\\
  \hline
  2 ($a_2$)& revenue ($x_2$)& -1\\
  \hline
 2 ($a_2$) & personal expenses  ($x_4$)& +1\\
 \hline
 3 ($a_3$)& other expenses ($x_6$) & +0.25\\
 \hline
  3 ($a_3$)& cost of sales ($x_3$) & +0.75\\
  \hline
 3 ($a_3$) & revenue ($x_2$)& -1\\
 \hline
 4 ($a_4$) & tax  ($x_1$) & -0.25\\
 \hline
 4 ($a_4$)& cost of sales ($x_3$) & +1\\
 \hline
 4 ($a_4$)& revenue ($x_2$) & -0.75\\
 \hline
 5 ($a_5$)& tax  ($x_1$) & -0.05\\
 \hline
 5 ($a_5$)& cost of sales ($x_3$) & +1\\
 \hline
 5 ($a_5$)& revenue ($x_2$)& -0.95\\
 \hline
 6 ($a_6$)& other expenses ($x_6$) & +0.1\\
 \hline
6 ($a_6$) & cost of sales ($x_3$)   & +0.9\\
\hline
 6 ($a_6$)& inventory ($x_5$)& -1\\
 \hline
 7 ($a_7$)& cost of sales ($x_3$) & +1\\
 \hline
  7 ($a_7$)& inventory ($x_5$) & -0.25\\
  \hline
 7 ($a_7$)& revenue  ($x_2$)& -0.75\\
 \hline
 8 ($a_8$) & revenue ($x_2$)& -1\\
 \hline
  8 ($a_8$)& cost of sales ($x_3$)& +1\\
  \hline
 9 ($a_9$)& revenue ($x_2$)& -1\\
 \hline
9 ($a_9$)& other expenses ($x_6$)&+1\\
\hline
  10 ($a_{10}$)& tax ($x_1$) & -1\\
  \hline
 10  ($a_{10}$) & personal expenses ($x_4$) & +1\\
 \hline
 11  ($a_{11}$)& revenue($x_2$) & -1\\
 \hline
 11 ($a_{11}$)& personal expenses ($x_4$) & +0.7\\
 \hline
 11 ($a_{11}$)& other  expenses ($x_6$) & +0.3\\
 \hline
 12 ($a_{12}$)& revenue($x_2$) & -0.83\\
 \hline
 12  ($a_{12}$)& tax($x_1$)  & -0.17\\
 \hline
12  ($a_{12}$)& personal expenses ($x_4$)  &+0.5\\
\hline
 12  ($a_{12}$)& other  expenses ($x_6$)  &+0.5\\
 \hline
\end{longtable}

 \subsection{B \ concept lattices obtained in different examples}\label{sec:lattice diagrams}
 
 In this Section, we give the diagrams of various concept lattices arising from the different examples in Section \ref{sec:Examples end}. These concept lattices were drawn with the help of lattice visualization by LatViz \cite{Latviz, alam2016latviz, alam2016steps, alam2016interactive}. 
 
 \begin{figure}[!h]
    \centering
    \includegraphics [scale=0.4]{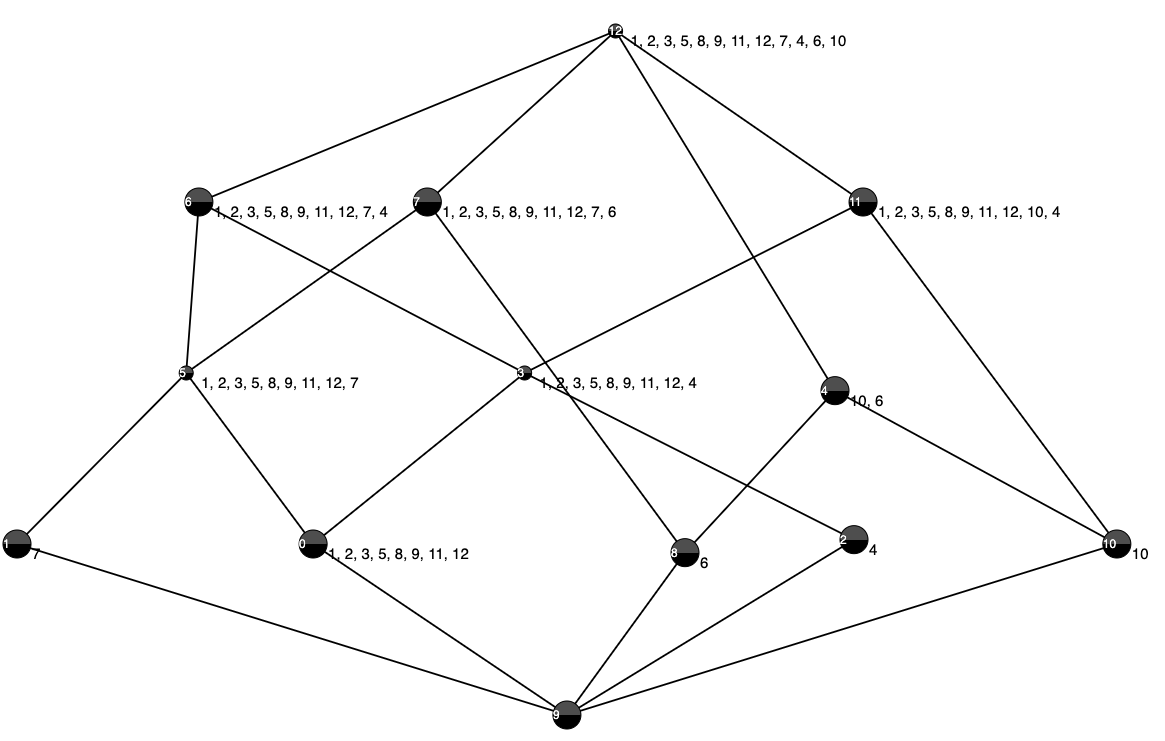}
    \caption{Concept lattice 1}
    \label{fig:lattice 1}
\end{figure}

\begin{figure}
    \centering
    \includegraphics [scale=0.4]{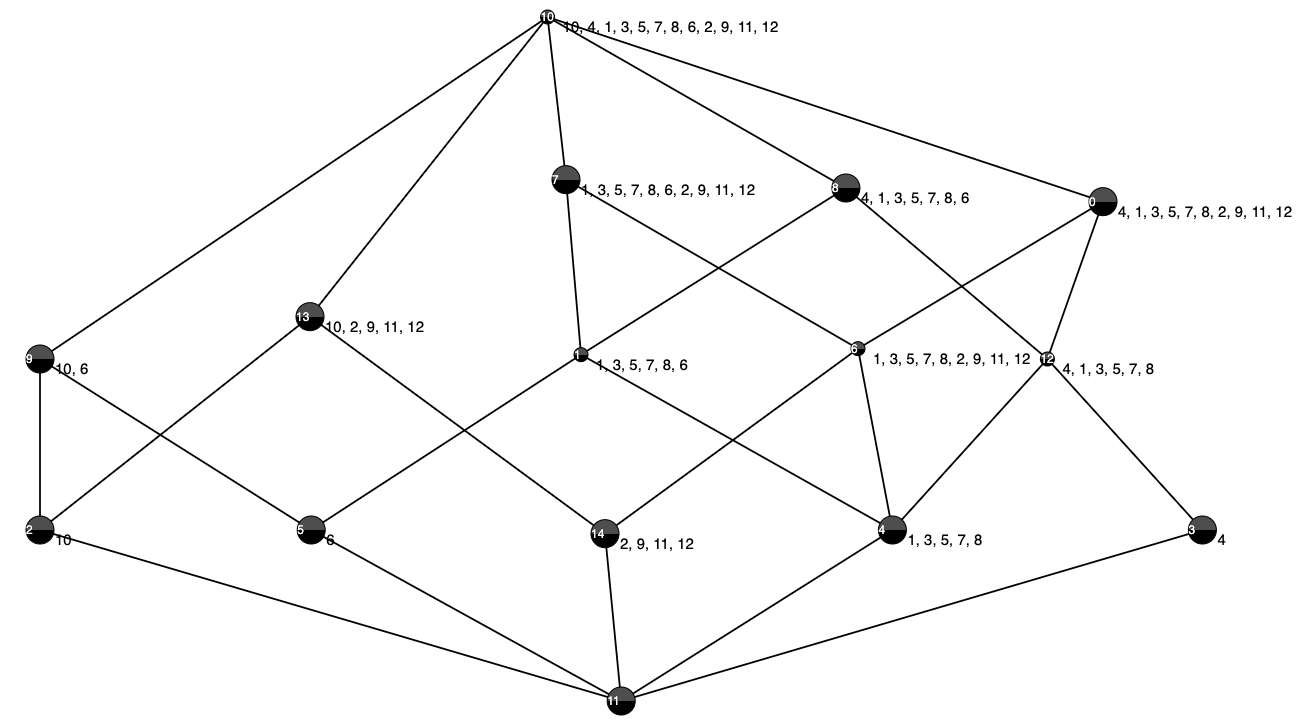}
    \caption{Concept lattice 2}
    \label{fig:lattice 3}
\end{figure}

\begin{figure}
    \centering
    \includegraphics [scale=0.4]{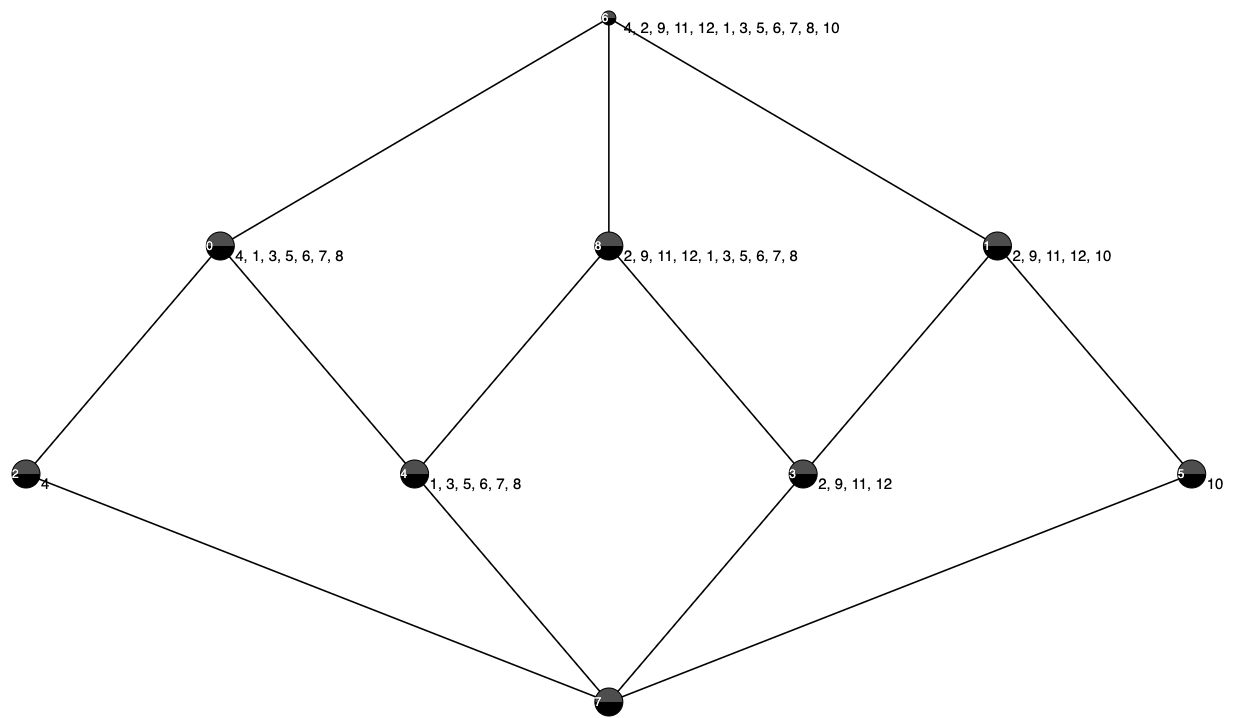}
    \caption{Concept lattice 3}
    \label{fig:lattice 5}
\end{figure}

\begin{figure}
    \centering
    \includegraphics[scale=0.4]{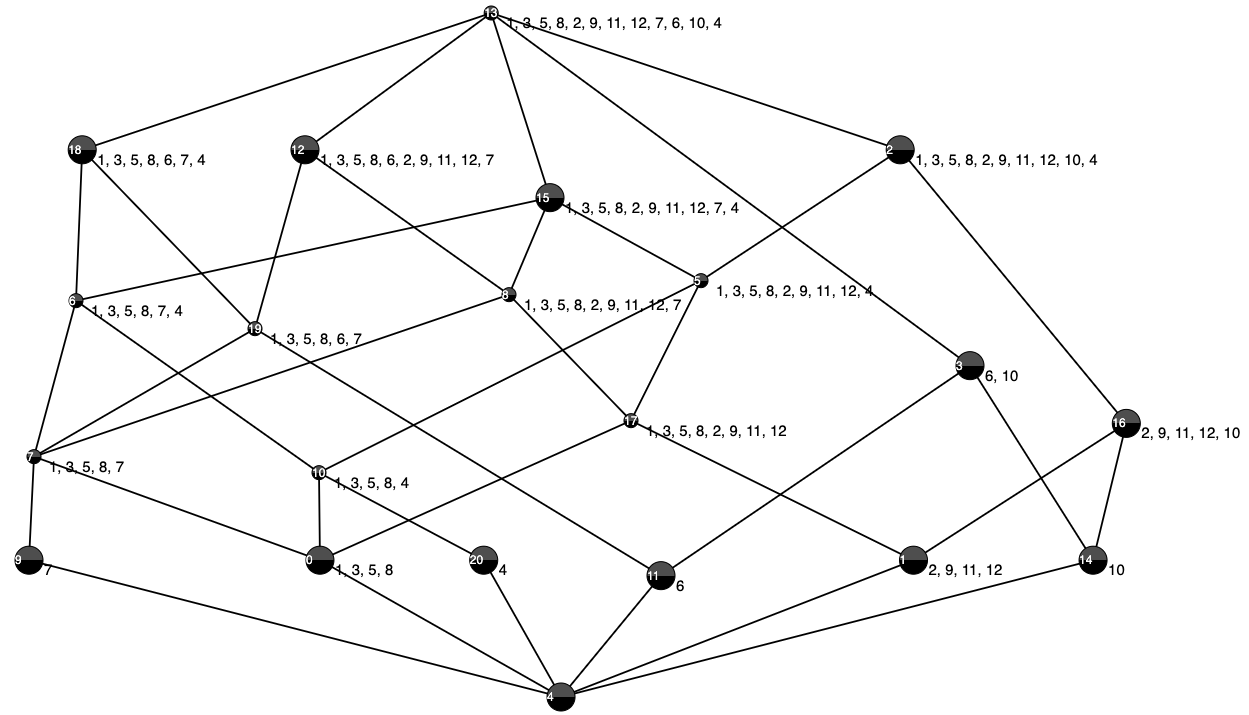}
    \caption{Concept lattice 4}
    \label{fig:lattice 2}
\end{figure}

\begin{figure}
    \centering
    \includegraphics[scale=0.4]{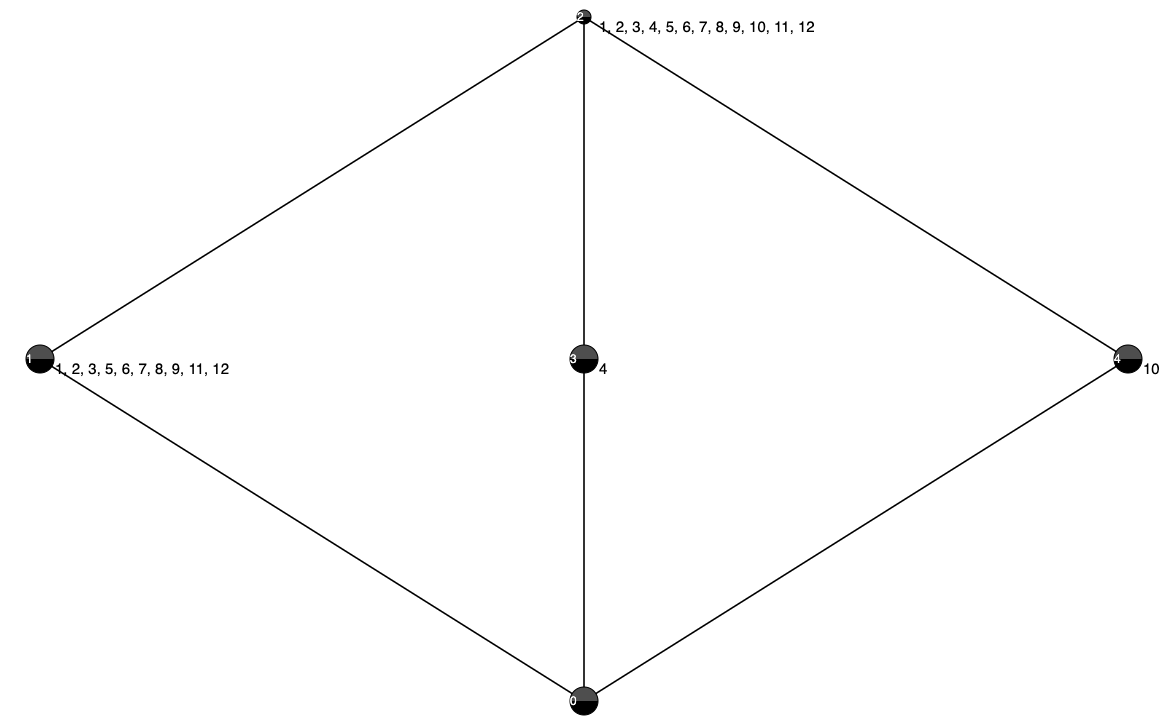}
    \caption{Concept lattice 5}
    \label{fig:lattice 6}
\end{figure}

\begin{figure}
    \centering
    \includegraphics[scale=0.385]{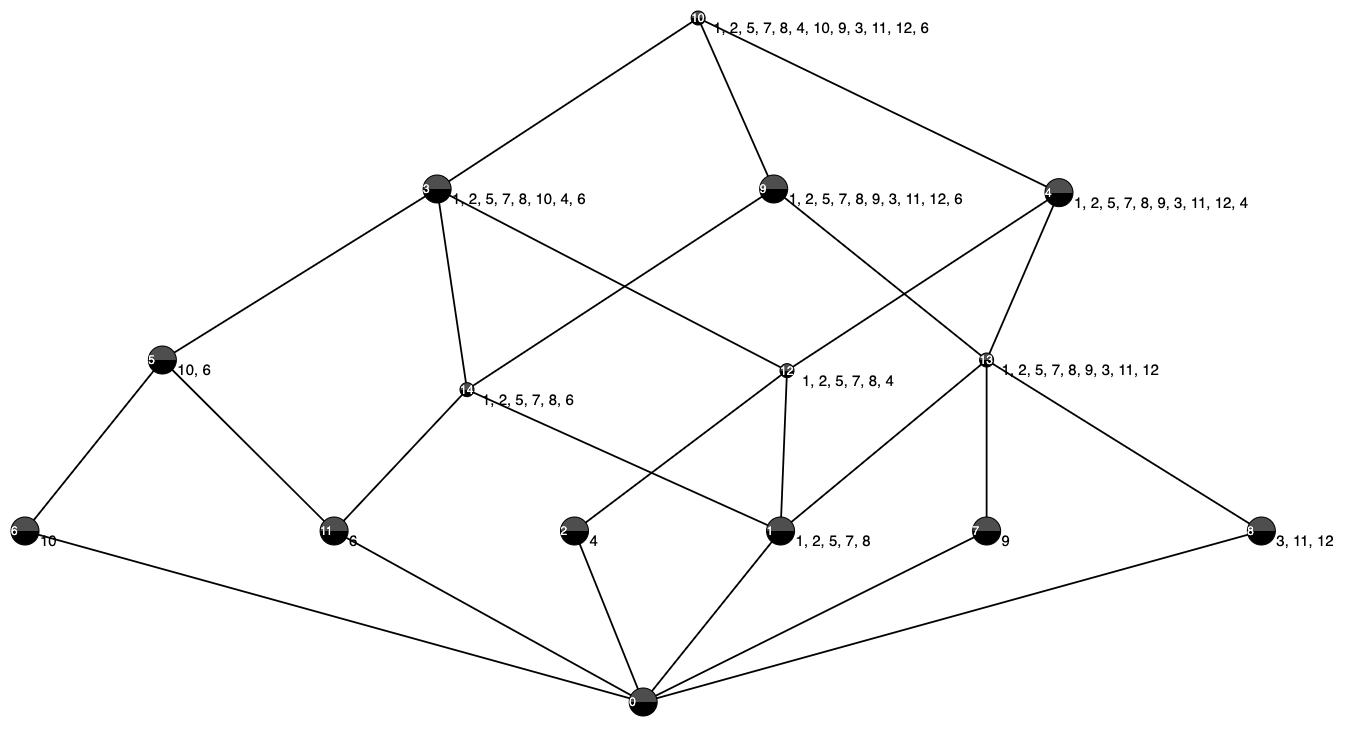}
    \caption{Concept lattice 6}
    \label{fig:lattice 7}
\end{figure}

\begin{figure}
    \centering
    \includegraphics[scale=0.4]{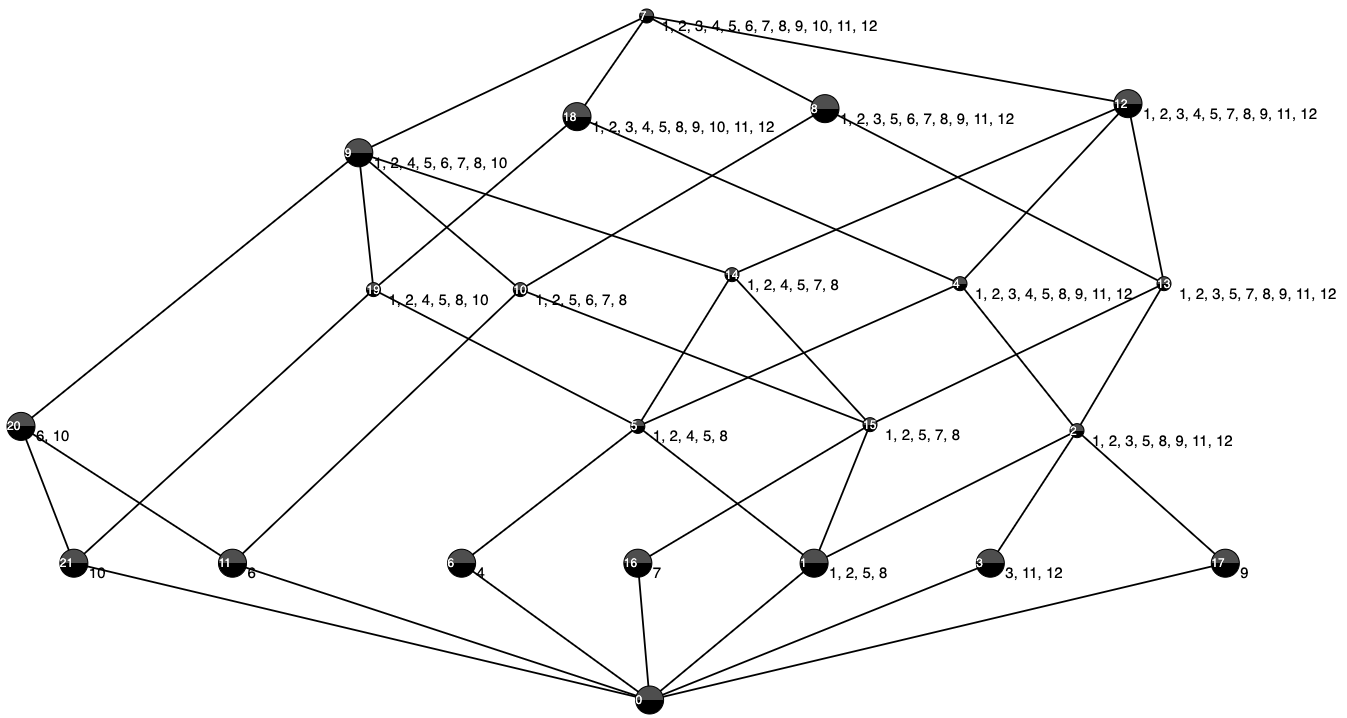}
    \caption{Concept lattice 7}
    \label{fig:lattice 9}
\end{figure}

\begin{figure}
    \centering
    \includegraphics [scale=0.4]{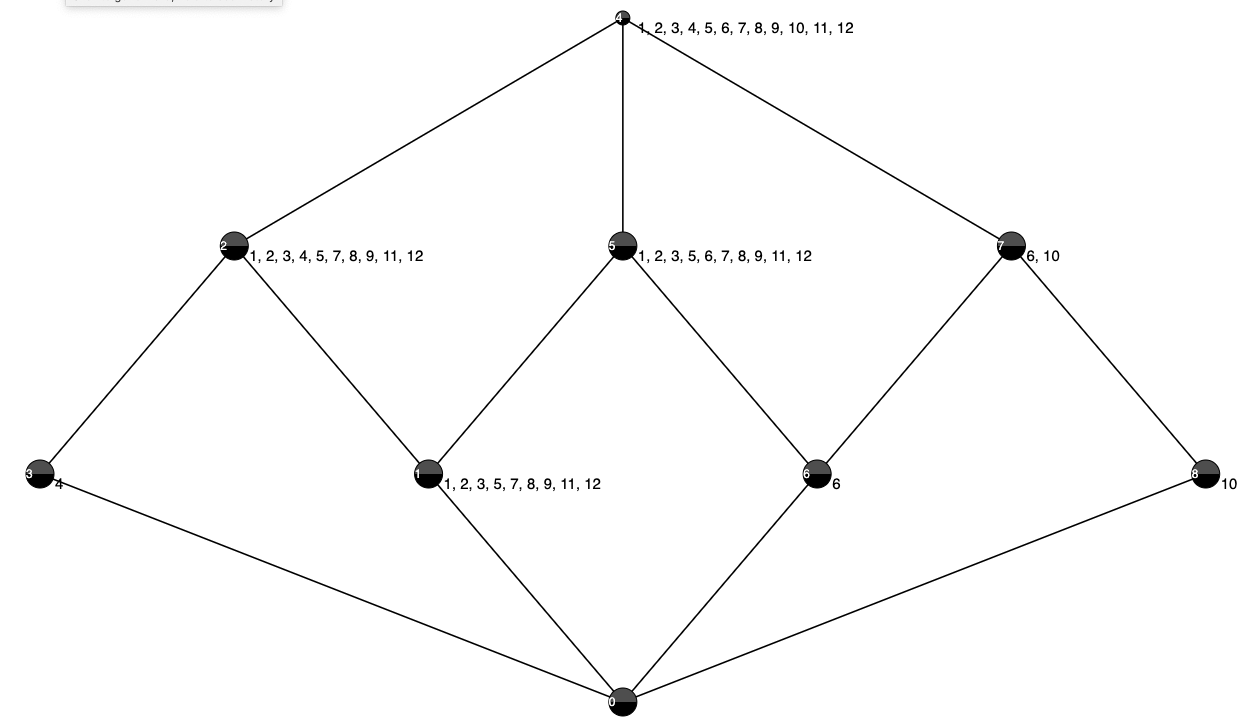}
    \caption{Concept lattice 8}
    \label{fig:lattice 4}
\end{figure}
\begin{figure}
    \centering
    \includegraphics[scale=0.4]{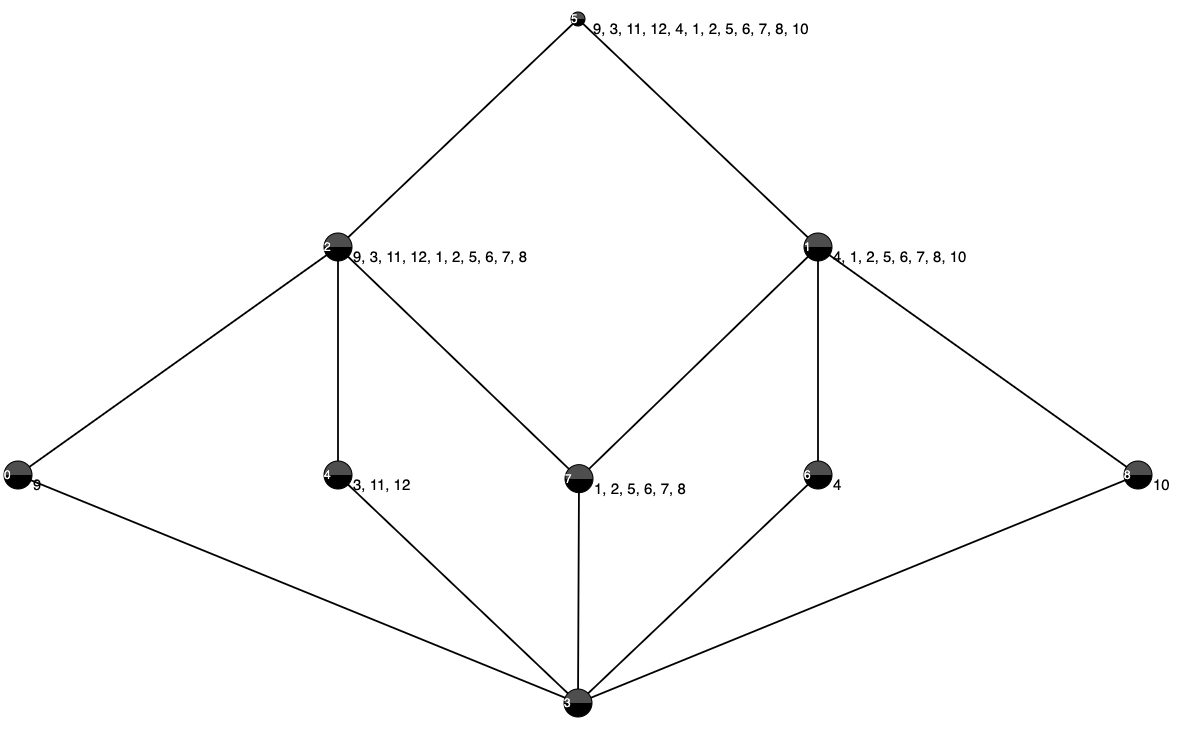}
    \caption{Concept lattice 9}
    \label{fig:lattice 8}
\end{figure}
\begin{figure}
    \centering
    \includegraphics[scale=0.4]{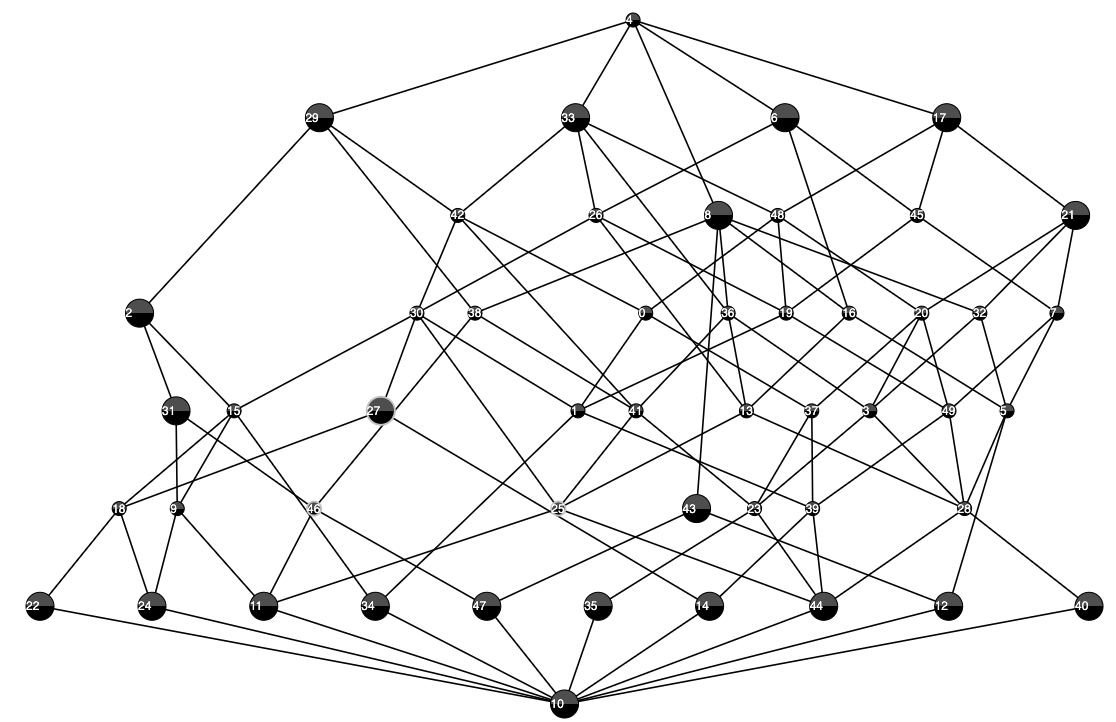}
    \caption{Concept lattice 10}
    \label{fig:lattice 10}
\end{figure}

\end{document}